\documentclass[draftclsnofoot,onecolumn]{IEEEtran}%

\usepackage{hyperref} 
\usepackage{graphicx,xspace}
\usepackage[noend]{algorithm,algorithmic}
\usepackage{pstricks,pst-node}
\usepackage{amsmath,amssymb,amsfonts,color,psfrag}
\usepackage{subfigure}
\usepackage{epsfig}
\usepackage{multicol} 
\usepackage{multirow} 

\newtheorem{theorem}{Theorem}[section]
\newtheorem{lemma}[theorem]{Lemma}
\newtheorem{remark}[theorem]{Remark}

\newtheorem{definition}[theorem]{Definition}

\newcommand{\real}{{\mathbb{R}}}

\newcommand{\scirc}{\raise1pt\hbox{$\,\scriptstyle\circ\,$}}

\newcommand{\PP}{\mathcal{P}}

\newcommand{\MM}{\mathcal{M}}
\newcommand{\EE}{\mathcal{E}}

\newcommand{\TT}{\mathcal{T}}

\newcommand{\VV}{\mathcal{V}} 

\newcommand{\interior}{\operatorname{int}}

\newcommand{\Gvis}[1]{ \mathcal{G}_{{\rm vis},#1} }  
\newcommand{\umax}{u_{\rm max}}             

\newcommand{\successor}{\operatorname{successor}}

\newcommand{\Tp}{{\TT}_{\PP}}     
\newcommand{\VVver}{\tilde{\mathcal{V}}} 

\newcommand{\LISTEN}{\textup{LISTEN}}

\newcommand{\PROCESS}{\textup{PROCESS}}

\newcommand{\SLEW}{\textup{SLEW}}

\newcommand{\BROADCAST}{\textup{BROADCAST}}

\newcommand\oprocendsymbol{\hbox{$\square$}}
\newcommand\oprocend{\relax\ifmmode\else\unskip\hfill\fi\oprocendsymbol}

\begin{document}\renewcommand{\thefootnote}{\fnsymbol{footnote}}



  
 
\title{\LARGE \bf Multi-Agent~Deployment~for Visibility~Coverage~in
  Polygonal~Environments~with~Holes \thanks{
    \noindent This version: \today.  This work has been supported in
    part by AFOSR MURI Award F49620-02-1-0325, NSF Award CMS-0626457,
    and a DoD SMART fellowship.  Thanks to Michael Schuresko (UCSC)
    and Antonio Franchi (Uni Roma) for helpful comments.}}
\author{
  Karl J. Obermeyer \thanks{Karl J. Obermeyer and Francesco Bull are
    with the Center for Control, Dynamical Systems, and Computation,
    University of California at Santa Barbara, Santa Barbara, CA
    93106, USA, \texttt{karl@engineering.ucsb.edu},
    \texttt{bullo@engineering.ucsb.edu}}
  \and \ \ \ \ \ \ \ \ \ \ \ \ \ \ Anurag Ganguli\thanks{Anurag
    Ganguli is with the UtopiaCompression Corporation, 11150 W. Olympic
    Blvd, Suite 820 Los Angeles, CA 90064,
    \texttt{anurag@utopiacompression.com}}
  \and \ \ \ \ \ \ \ \ \ \ \ \ \ \ Francesco
  Bullo
 }

\maketitle

\begin{abstract}
  This article presents a distributed algorithm for a group of robotic
  agents with omnidirectional vision to deploy into nonconvex
  polygonal environments with holes.  Agents begin deployment from a
  common point, possess no prior knowledge of the environment, and
  operate only under line-of-sight sensing and communication.  The
  objective of the deployment is for the agents to achieve full
  visibility coverage of the environment while maintaining
  line-of-sight connectivity with each other.  This is achieved by
  incrementally partitioning the environment into distinct regions,
  each completely visible from some agent.  Proofs are given of (i)
  convergence, (ii) upper bounds on the time and number of agents
  required, and (iii) bounds on the memory and communication
  complexity. Simulation results and description of robust extensions
  are also included.
\end{abstract}
 
\section{Introduction}
\label{sec:intro}

Robots are increasingly being used for surveillance missions too
dangerous for humans, or which require duty cycles beyond human
capacity.  In this article we design a distributed algorithm for
deploying a group of mobile robotic agents with omnidirectional vision
into nonconvex polygonal environments with holes, e.g., an urban or
building floor plan.  Agents are identical except for their unique
identifiers (UIDs), begin deployment from a common point, possess no
prior knowledge of the environment, and operate only under
line-of-sight sensing and communication.  The objective of the
deployment is for the agents to achieve full visibility coverage of
the environment while maintaining line-of-sight connectivity (at any
time the agents' visibility graph consists of a single connected
component).  We call this the \emph{Distributed Visibility-Based
  Deployment Problem with Connectivity}.  Once deployed, the agents
may supply surveillance information to an operator through the ad-hoc
line-of-sight communication network.  A graphical description of our
objective is given in Fig.~\ref{fig:dfcd_sim}.

\begin{figure}[thb!]
\begin{center}
\resizebox{0.3\linewidth}{!}{ \includegraphics{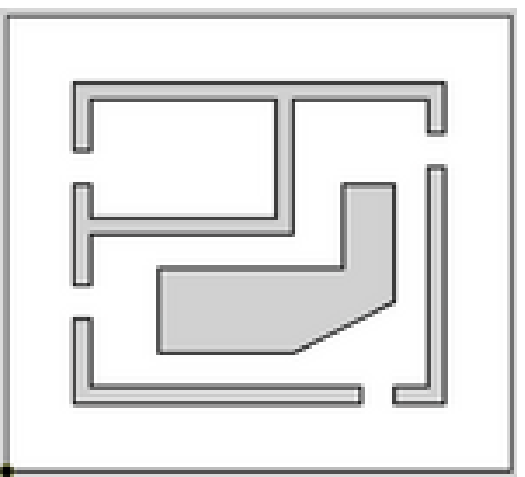} } 
\qquad \resizebox{0.3\linewidth}{!}{ \includegraphics{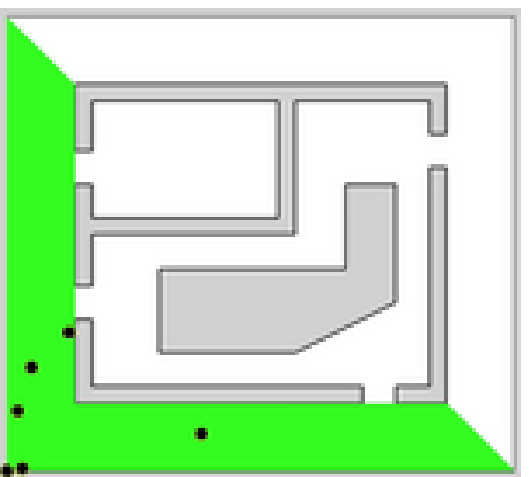} } 
\qquad \resizebox{0.3\linewidth}{!}{ \includegraphics{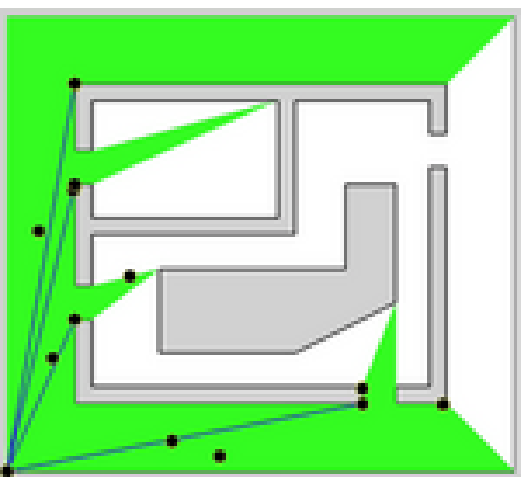} } 
\vspace{0.1em}

\resizebox{0.3\linewidth}{!}{ \includegraphics{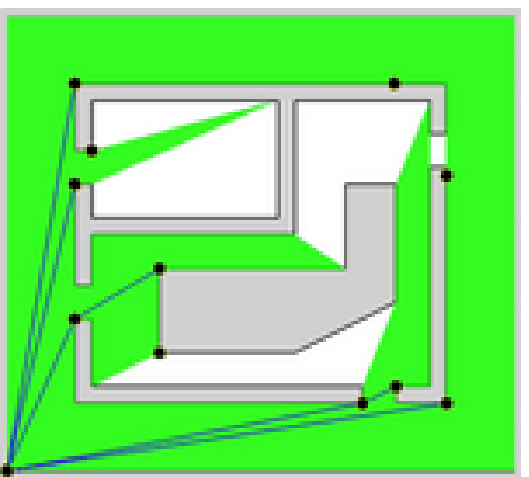} } 
\qquad \resizebox{0.3\linewidth}{!}{ \includegraphics{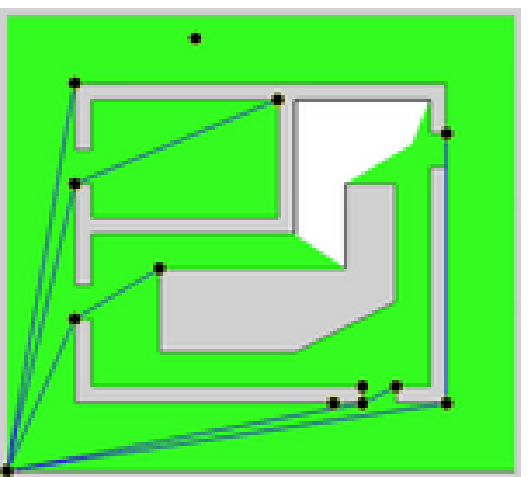} } 
\qquad \resizebox{0.3\linewidth}{!}{ \includegraphics{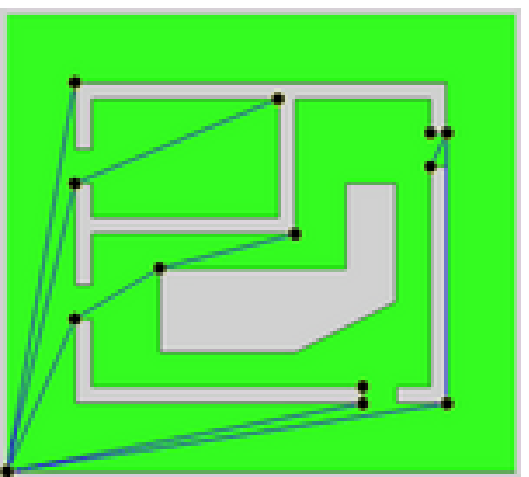} } 

\caption{\label{fig:dfcd_sim} This sequence (left to right, top to
  bottom) shows a simulation run of the distributed visibility-based
  deployment algorithm described in
  Sec.~\ref{sec:distributed_deployment}.  Agents (black disks)
  initially are colocated in the lower left corner of the environment.
  As the agents spread out, they claim areas of responsibility (green)
  which correspond to cells of the incremental partition tree $\Tp$.
  Blue lines show line-of-sight connections between agents responsible
  for neighboring vertices of $\Tp$.  Once agents have settled to
  their final positions, every point in the environment is visibile to
  some agent and the agents form a line-of-sight connected network.
  An animation of this simulation can be viewed at {\tt
    http://motion.me.ucsb.edu/$\sim$karl/movies/dwh.mov} .}
\end{center}
\end{figure}

Approaches to visibility coverage problems can be divided into two
categories: those where the environment is known a priori and those
where the environment must be discovered.
When the environment is known a priori, a well-known approach is the
\emph{Art Gallery Problem} in which one seeks the smallest set of
guards such that every point in a polygon is visible to some guard.
This problem has been shown both NP-hard~\cite{DTL-AKL:86} and
APX-hard~\cite{SE-CS-PW:01} in the number of vertices $n$ representing
the environment.  The best known approximation algorithms offer
solutions only within a factor of $O(\log g)$, where $g$ is the
optimum number of agents~\cite{AE-SHP:06}.  The \emph{Art Gallery
  Problem with Connectivity} is the same as the Art Gallery Problem,
but with the additional constraint that the guards' visibility graph
must consist of a single connected component, i.e., the guards must
form a connected network by line of sight.  This problem is also
NP-hard in $n$ \cite{BCL-NFH-RCTL:93}.  Many other variations on the
Art Gallery Problem are well surveyed in \cite{JU:00,JOR:87,TCS:92}.
The classical \emph{Art Gallery Theorem}, proven first in \cite{VC:75}
by induction and in \cite{SF:78} by a beautiful coloring argument,
states that $\lfloor \frac{n}{3} \rfloor$ vertex guards\footnote{A
  \emph{vertex guard} is a guard which is located at a vertex of the
  polygonal environment.}  are always sufficient and sometimes
necessary to cover a polygon with $n$ vertices and no holes.  The
\emph{Art Gallery Theorem with Holes}, later proven independently by
\cite{IBS-DS:95} and \cite{FH-MK-KK:91}, states that $\lfloor
\frac{n+h}{3} \rfloor$ point guards\footnote{A \emph{point guard} is a
  guard which may be located anywhere in the interior or on the
  boundary of a polygonal environment.} are always sufficient and
sometimes necessary to cover a polygon with $n$ vertices and $h$
holes.  If guard connectivity is required, \cite{GHP:94} proved by
induction and \cite{VP:03c} by a coloring argument, that $\lfloor
\frac{n-2}{2} \rfloor$ vertex guards are always sufficient and
occasionally necessary for polygons without holes.  We are not aware
of any such bound for connected coverage of polygons with holes.  For
polygonal environments with holes, centralized camera-placement
algorithms described in \cite{HGB-JCL:01} and \cite{ME-SS:06} take
into account practical imaging limitations such as camera range and
angle-of-incidence, but at the expense of being able to obtain
worst-case bounds as in the Art Gallery Theorems.  The constructive
proofs of the Art Gallery Theorems rely on global knowledge of the
environment and thus are not amenable to emulation by distributed
algorithms.

One approach to visibiliy coverage when the environment must be
discovered is to first use SLAM (Simultaneous Localization And
Mapping) techniques \cite{ST-WB-DF:05} to explore and build a map of
the entire environment, then use a centralized procedure to decide
where to send agents.  In~\cite{RS-DA-DF-RPG-KZH-DJM-MP-ST:00}, for
example, deployment locations are chosen by a human user after an
initial map has been built. Waiting for a complete map of the entire
environment to be built before placing agents may not be desirable.
In \cite{AH-MJM-GSS:02b} agents fuse sensor data to build only a map
of the portion of the environment covered so far, then heuristics are
used to deploy agents onto the frontier of the this map, thus
repeating this procedure incrementally expands the covered region.
For any techniques relying heavily on SLAM, however, synchronization
and data fusion can pose significant challenges under communication
bandwidth limitations.  In \cite{SS-EV-PW:08} agents discover and
achieve visibility coverage of an environment not by building a
geometric map, but instead by sharing only combinatorial information
about the environment, however, the strategy focuses on the
theoretical limits of what can be achieved with minimalistic sensing,
thus the amount of robot motion required becomes impractical.

Most relevant to and the inspiration for the present work are the
distributed visibility-based deployment algorithms, for polygonal
environments without holes, developed recently by Ganguli et al
\cite{AG-JC-FB:05z,AG-JC-FB:06r,AG:07}
.  These algorithms are simple, require only limited impact-based
communication, and offer worst-case optimal bounds on the number of
agents required.  The basic strategy is to incrementally construct a
so-called \emph{nagivation tree} through the environment.  To each
vertex in the navigation tree corresponds a region of the the
environment which is completely visible from that vertex.  As agents
move through the environment, they eventually settle on certain nodes
of the navigation tree such that the entire environment is covered.


The contribution of this article is the first distributed deployment
algorithm which solves, with provable performance, the Distributed
Visibility-Based Deployment Problem with Connectivity in polygonal
environments with holes.  Our algorithm operates using line-of-sight
communication and a so-called \emph{partition tree} data structure
similar to the \emph{navigation tree} used by Ganguli et al as
described above.  The algorithms of Ganguli et al fail in polygonal
environments with holes because branches of the navigation tree
conflict when they wrap around one or more holes.  Our algorithm,
however, is able to handle such ``branch conflicts''.  Given at least
$\lfloor \frac{n+2h-1}{2} \rfloor $ agents in an environment with $n$
vertices and $h$ holes, the deployment is guaranteed to achieve full
visibility coverage of the environment in time $\mathcal{O}(n^2+nh)$,
or time $\mathcal{O}(n+h)$ under certain technical conditions.  We
also prove bounds on the memory and communication complexity.  The
deployment behaves in simulations as predicted by the theory and can
be extended to achieve robustness to agent arrival, agent failure,
packet loss, removal of an environment edge (such as an opening door),
or deployment from multiple roots.

This article is organized as follows.  We begin with some technical
definitions in Sec.~\ref{sec:notation}, then a precise statement of
the problem and assumptions in Sec.~\ref{sec:problem}.  Details on the
agents' sensing, dynamics, and communication are given in
Sec.~\ref{sec:net_model}.  Algorithm descriptions, including
pseudocode and simulation results, are presented in
Sec.~\ref{sec:incremental_partition} and
Sec.~\ref{sec:distributed_deployment}.  We conclude in
Section~\ref{sec:conclusion}.

\section{Notation and Preliminaries}
\label{sec:notation}

We begin by introducing some basic notation. The real numbers are
represented by $\real$.  Given a set, say $A$, the interior of $A$ is
denoted by $\interior(A)$, the boundary by $\partial A$, and the
cardinality by $|A|$.  Two sets $A$ and $B$ are \emph{openly disjoint}
if $\interior(A) \cap \interior(B) = \emptyset$.  Given two points
$a,b \in \real^2$, $[a,b]$ is the \emph{closed segment} between $a$
and $b$.  Similarly, $]a,b[$ is the \emph{open segment} between $a$
and $b$.  The number of robotic agents is $N$ and each of these agents
has a unique identifier (UID) taking a value in $\{ 0,\ldots,N-1 \}$.
Agent positions are $P = (p^{[0]},\ldots, p^{[N-1]})$, a tuple of
points in $\real^2$.  Just as $p^{[i]}$ represents the position of
agent $i$, we use such superscripted square brackets with any variable
associated with agent $i$, e.g., as in Table~\ref{tab:dfcd_local_vars}.

We turn our attention to the environment, visibility, and graph
theoretic concepts.  The environment $\EE$ is polygonal with vertex
set $V_\EE$, edge set $E_\EE$, total vertex count $n=|V_\EE|=|E_\EE|$,
and hole count $h$.  Given any polygon $c \subset \EE$, the vertex set
of $c$ is $V_c$ and the edge set is $E_c$.  A segment $[a,b]$ is a
\emph{diagonal} of $\EE$ if (i) $a$ and $b$ are vertices of $\EE$, and
(ii) $]a,b[ \subset \interior(\EE)$.  Let $e$ be any point in $\EE$.
The point $e$ is \emph{visible from} another point $e'\in \EE$ if
$[e,e'] \subset \EE$.  The \emph{visibility polygon} $\VV(e)\subset
\EE$ of $e$ is the set of points in $\EE$ visible from $e$
(Fig.~\ref{fig:visibility_polygons}).  The \emph{vertex-limited
  visibility polygon} $\VVver(e) \subset \VV$ is the visibility
polygon $\VV(e)$ modified by deleting every vertex which does not
coincide with an environment vertex
(Fig.~\ref{fig:visibility_polygons}).  A \emph{gap edge} of $\VV(e)$
(resp. $\VVver(e)$) is defined as any line segment $[a,b]$ such that
$]a,b[ \subset \interior(\EE)$, $[a,b] \subset \partial \VV(e)$
(resp. $[a,b] \subset \partial \VVver(e)$), and it is maximal in the
sense that $a,b \in \partial \EE$.  Note that a gap edge of
$\VVver(e)$ is also a diagonal of $\EE$.  For short, we refer to the
gap edges of $\VV(e)$ as the \emph{visibility gaps} of $e$.
\begin{figure}[h!]
\begin{center}
  \resizebox{0.4 \linewidth}{!}{\includegraphics{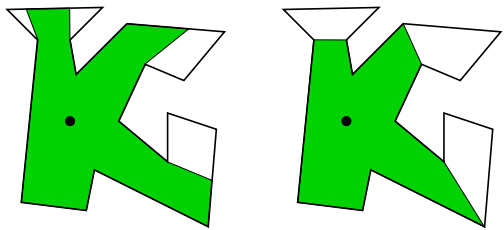}}
  \caption{\label{fig:visibility_polygons} In a simple nonconvex
    polygonal environment are shown examples of the visibility polygon
    (red, left) of a point observer (black disk), and the vertex-limited
    visibility polygon (red, right) of the same point.}
\end{center}
\end{figure}
A set $R \subset \EE$ is \emph{star-convex} if there exists a point $e
\in R $ such that $R \subset \VV(e)$.  The \emph{kernel} of a
star-convex set $R$, 
is the set $\{e \in \EE | R \subset \VV(e) \}$, i.e., all points in
$R$ from which all of $R$ is visible.  The \emph{visibility graph}
$\Gvis{\EE}(P)$ of a set of points $P$ in environment $\EE$ is the
undirected graph with $P$ as the set of vertices and an edge between
two vertices if and only if they are (mutually) visible.  A
\emph{tree} is a connected graph with no simple cycles.  A
\emph{rooted tree} is a tree with a special vertex designated as the
\emph{root}.  The \emph{depth} of a vertex in a rooted tree is the
minimum number of edges which must be treversed to reach the root from
that vertex. Given a tree $\TT$, $V_\TT$ is its set of vertices and
$E_\TT$ its set of edges.

\section{Problem Description and Assumptions}
\label{sec:problem}

The \emph{Distributed Visibility-Based Deployment Problem with
  Connectivity} which we solve in the present work is formally stated
as follows:
\begin{quote}
  Design a distributed algorithm for a network of autonomous robotic
  agents to deploy into an unmapped environment such that from their
  final positions every point in the environment is visible from some
  agent.  The agents begin deployment from a common point, their
  visibility graph $\Gvis{\EE}(P)$ is to remain connected, and they
  are to operate using only information from local sensing and
  line-of-sight communication.
\end{quote}
By local sensing we intend that each agent is able to sense its
visibility gaps and relative positions of objects within line of
sight.  Additionally, we make the following \emph{main assumptions}:
\begin{enumerate}
\vspace{0.2em}
\item The environment $\EE$ is static and consists of a simple
  polygonal outer boundary together with disjoint simple polygonal
  holes.  By simple we mean that each polygon has a single boundary
  component, its boundary does not intersect itself, and the number of
  edges is finite.  
\vspace{0.5em}
\item Agents are identical except for their UIDs ($0,\ldots,N-1$).
  \vspace{0.5em}
\item Agents do not obstruct visibility or movement of other agents.
  \vspace{0.5em}
\item Agents are able to locally establish a common reference frame.
  \vspace{0.5em}
\item There are no communication errors nor packet losses.  
  \vspace{0.5em}
\end{enumerate}

Later, in Sec.~\ref{subsec:extensions} we will describe how our
nominal deployment algorithm can be extended to relax some
assumptions.

\section{Network of Visually-Guided Agents}
\label{sec:net_model}



In this section we lay down the sensing, dynamic, and communication
model for the agents.  Each agent has ``omnidirectional vision''
meaning an agent possesses some device or combination of devices which
allows it to sense within line of sight (i) the relative position of
another agent, (ii) the relative position of a point on the boundary
of the environment, and (iii) the gap edges of its visibility polygon.

For simplicity, we model the agents as point masses with first order
dynamics, i.e., agent $i$ may move through $\EE$ according to the
continuous time control system
\begin{equation}
\label{eqn:agent_dynamics}
\dot{p}^{[i]} = u^{[i]},
\end{equation}
where the control $u^{[i]}$ is bounded in magnitude by $\umax$.  The
control action depends on time, values of variables stored in local
memory, and the information obtained from communication and sensing.
Although we present our algorithms using these first order dynamics,
the crucial property for convergence is only that an agent is able to
navigate along any (unobstructed) straight line segment between two
points in the environment $\EE$, thus the deployment algorithm we
describe is valid also for higher order dynamics.
%

The agents' communication graph is precisely their visibility graph
$\Gvis{\EE}(P)$, i.e., any \emph{visibility neighbors} (mutually
visible agents) may communicate with each other.  Agents may send
their messages using, e.g., UDP (User Datagram Protocol).  Each agent
($i=0,\ldots,N-1$) stores received messages in a FIFO
(First-In-First-Out) buffer In\_Buffer$^{[i]}$ until they can be
processed.  Messages are sent only upon the occurrence of certain
asynchronous events and the agents' processors need not be
synchronized, thus the agents form an \emph{event-driven asynchronous
  robotic network} similar to that described, e.g., in
\cite{FB-JC-SM:09}.
In order for two visibility neighbors to establish a common reference
frame, we assume agents are able to solve the \emph{correspondence
  problem}: the ability to associate the messages they receive with
the corresponding robots they can see.  This may be accomplished,
e.g., by the robots performing localization, however, as mentioned in
Sec.~\ref{sec:intro}, this might use up limited communication
bandwidth and processing power.  Simpler solutions include having
agents display different colors, ``license plates'', or periodic
patterns from LEDs \cite{DC-JM-BP-OAAO-YC-RF:07}.
%

\section{Incremental Partition Algorithm}
\label{sec:incremental_partition}

We introduce a centralized algorithm to incrementally partition the
environment $\EE$ into a finite set of openly disjoint star-convex
polygonal cells.  Roughly, the algorithm operates by choosing at each
step a new \emph{vantage point} on the frontier of the uncovered
region of the environment, then computing a cell to be covered by that
vantage point (each vantage point is in the kernel of its
corresponding cell).  The frontier is pushed as more and more vantage
point - cell pairs are added until eventually the entire environment
is covered.  The vantage point - cell pairs form a directed rooted
tree structure called the \emph{partition tree} $\Tp$.  This algorithm
is a variation and extension of an incremental partition algorithm
used in \cite{AG:07}, 
the main differences being that we have added a protocol for handling
holes and adapted the notation to better fit the added complexity of
handling holes.  The deployment algorithm to be described in
Sec.~\ref{sec:distributed_deployment} is a distributed emulation of
the centralized incremental partition algorithm we present here.

Before examining the precise pseudocode
Table~\ref{tab:incremental_partition}, we informally step through the
incremental partition algorithm for the simple example of
Fig.~\ref{fig:incremental_partition}a-f.  This sequence shows the
environment partition together with corresponding abstract
representations of the partition tree $\Tp$.  Each vertex of $\Tp$ is
a vantage point - cell pair and edges are based on cell adjacency.
Given any vertex of $\Tp$, say $(p_\xi, c_\xi)$, $\xi$ is the
\emph{PTVUID (Partition Tree Vertex Unique IDentifier)}.  The PTVUID
of a vertex at depth $d$ is a $d$-tuple, e.g., (1), (2,1), or (1,1,1).
The symbol $\emptyset$ is used as the root's PTVUID.  The algorithm
begins with the root vantage point $p_\emptyset$.  The cell of
$p_\emptyset$ is the grey shaded region $c_\emptyset$ in
Fig.~\ref{fig:incremental_partition}a, which is the vertex-limited
visibility polygon $\VVver(p_\emptyset)$.  According to certain
technical criteria, made precise later, child vantage points are
chosen on the endpoints of the unexplored gap edges.  In
Fig.~\ref{fig:incremental_partition}a, dashed lines show the
unexplored gap edges of $c_\emptyset$.  Selecting $p_{(1)}$ as the
next vantage point, the corresponding cell $c_{(1)}$ becomes the
portion of $\VVver(p_{(1)})$ which is across the parent gap edge and
extends away from the parent's cell.  The vantage point $p_{(2)}$ and
its cell $c_{(2)}$ are generated in the same way.  There are now three
vertices, $(p_\emptyset, c_\emptyset)$, $(p_{(1)}, c_{(1)})$, and
$(p_{(2)}, c_{(2)})$ in $\Tp$ (Fig.~\ref{fig:incremental_partition}b).
In a similar manner, two more vertices, $(p_{(2,1)}, c_{(2,1)})$ and
$(p_{(2,1,1)}, c_{(2,1,1)})$, have been added in
Fig.~\ref{fig:incremental_partition}c.  An intersection of positive
area is found between cell $c_{(2,1,1)}$ and the cell of another
branch of $\Tp$, namely $c_{(1)}$.  To solve this \emph{branch
  conflict}, the cell $c_{(2,1,1)}$ is discarded and a special marker
called a \emph{phantom wall} (thick dashed line in
Fig.~\ref{fig:incremental_partition}d) is placed where its parent gap
edge was.  A phantom wall serves to indicate that no branch of $\Tp$
should cross a particular gap edge. The vertex $(p_{(1,2)},
c_{(1,2)})$ added in Fig.~\ref{fig:incremental_partition}e thus can
have no children.  Finally, Fig.~\ref{fig:incremental_partition}f
shows the remaining vertices $(p_{(1,1)}, c_{(1,1)})$ and
$(p_{(1,1,1)}, c_{(1,1,1)})$ added to $\Tp$ so that the entire
environment is covered and the algorithm terminates.

\begin{figure}[t!]
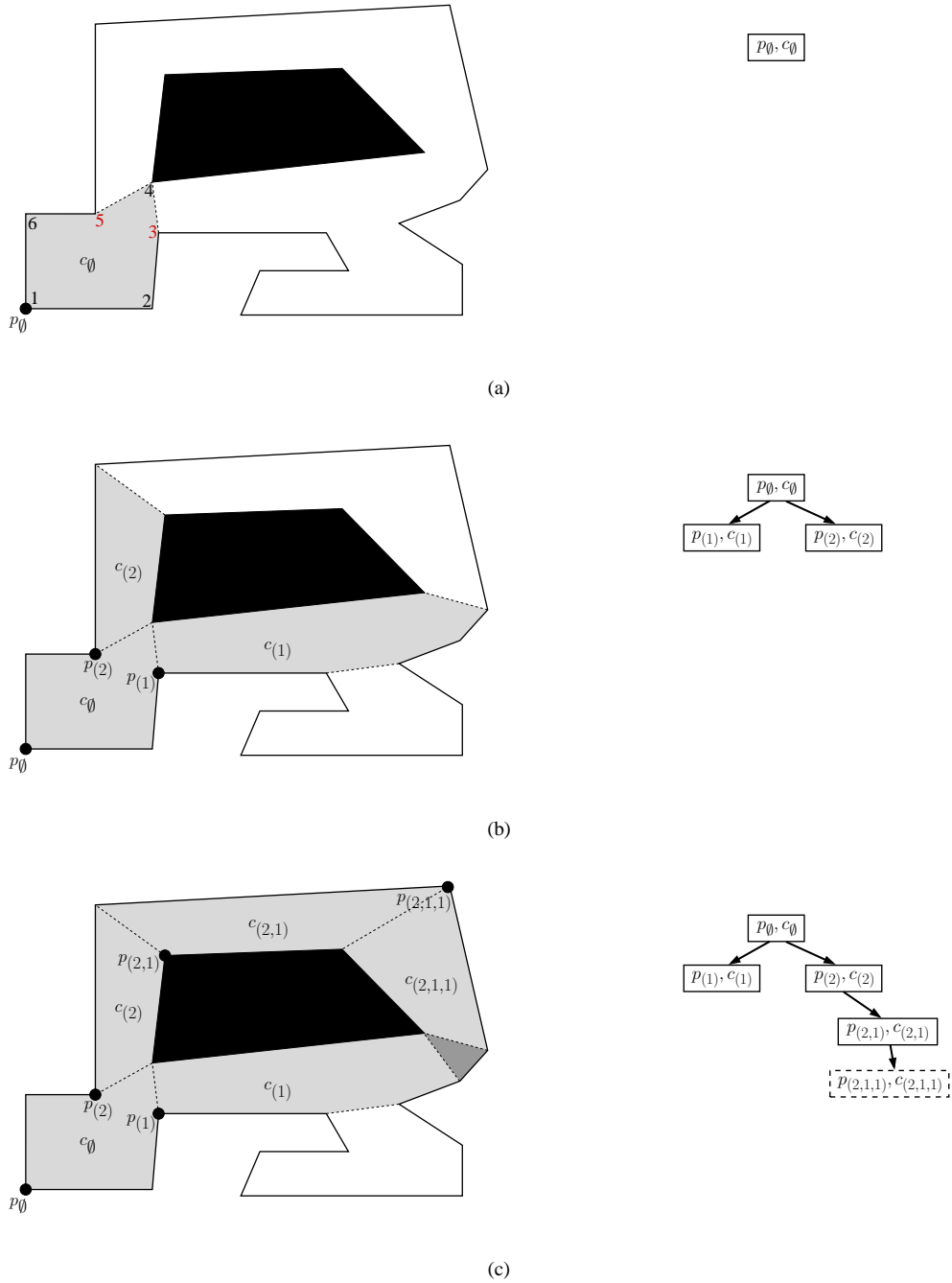

\begin{center}
   \subfigure[]{
    \resizebox{0.45\linewidth}{!}{\input{./fig/incremental_partition1.tex}} \hspace{2em}
    \resizebox{0.36\linewidth}{!}{\input{./fig/abstract_partition_tree1.tex}}
  }
   \subfigure[]{
    \resizebox{0.45\linewidth}{!}{\input{./fig/incremental_partition3.tex}} \hspace{2em}
    \resizebox{0.36\linewidth}{!}{\input{./fig/abstract_partition_tree3.tex}}
  }
   \subfigure[]{
    \resizebox{0.45\linewidth}{!}{\input{./fig/incremental_partition5a.tex}} \hspace{2em}
    \resizebox{0.36\linewidth}{!}{\input{./fig/abstract_partition_tree5a.tex}}
  }
  \caption{\label{fig:incremental_partition} This simple example shows
    how the incremental partition algorithm of
    Table~\ref{tab:incremental_partition} progresses (a)-(f).  Cell
    vantage points are shown by black disks.  The portion of the
    environment $\EE$ covered at each stage is shown in grey (left)
    along with a corresponding abstract depiction of the partition
    tree (right).  A phantom wall (thick dashed line), shown first in
    (d), comes about when there is a \emph{branch conflict}, i.e.,
    when cells from different branches of the partition tree $\Tp$ are
    not openly disjoint.  The final partition can be used to
    triangulate the environment as shown in
    Fig.~\ref{fig:triangulation}.}
\end{center}
\end{figure}

\begin{figure}[t!]
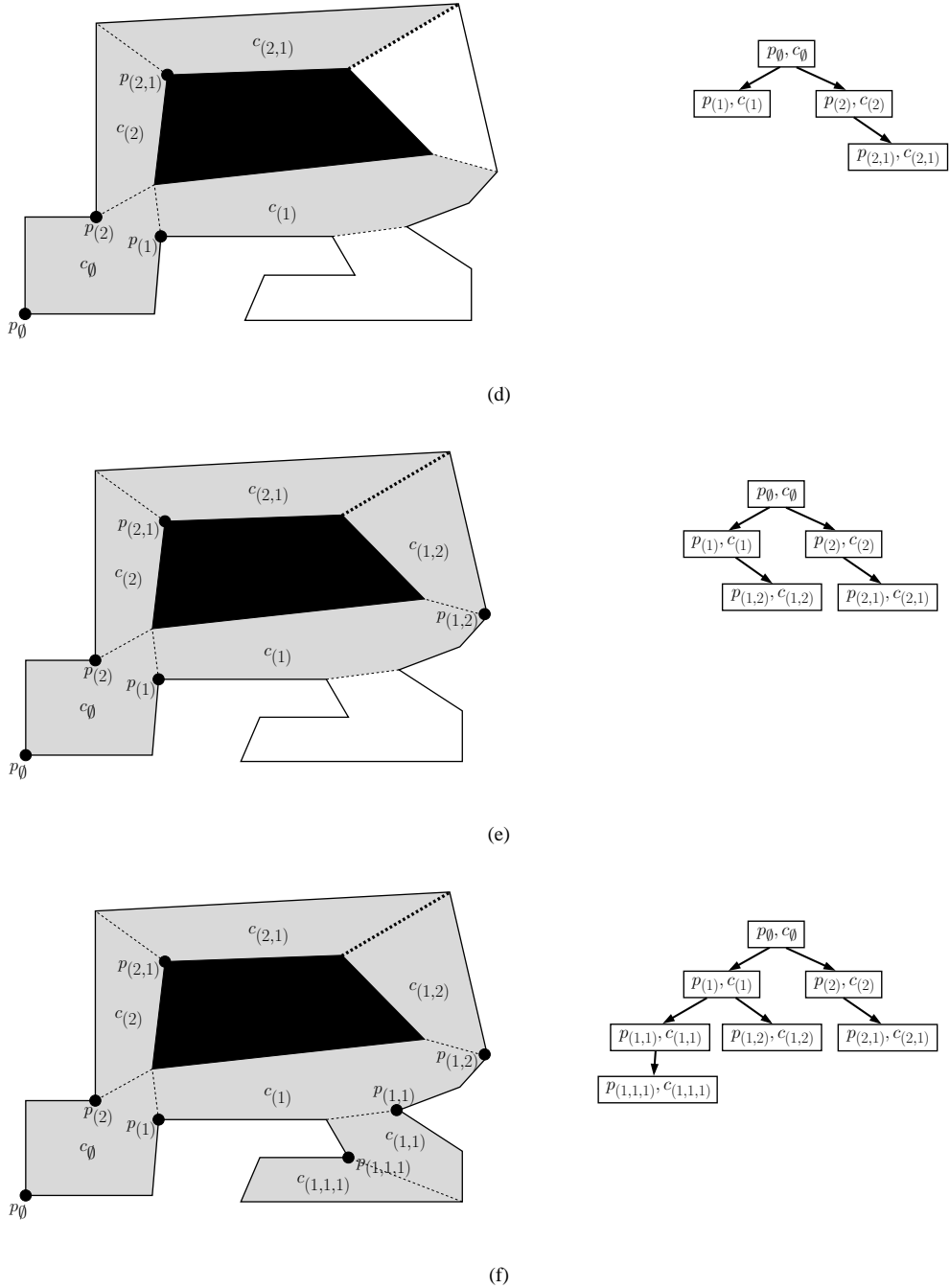

\addtocounter{figure}{-1}
\begin{center}
   \subfigure[]{
     \addtocounter{subfigure}{3}
    \resizebox{0.46\linewidth}{!}{\input{./fig/incremental_partition5b.tex}} \hspace{2em}
    \resizebox{0.36\linewidth}{!}{\input{./fig/abstract_partition_tree5b.tex}}
  }
   \subfigure[]{
    \resizebox{0.45\linewidth}{!}{\input{./fig/incremental_partition5c.tex}} \hspace{2em}
    \resizebox{0.36\linewidth}{!}{\input{./fig/abstract_partition_tree5c.tex}}
  }
   \subfigure[]{
    \resizebox{0.45\linewidth}{!}{\input{./fig/incremental_partition7.tex}} \hspace{2em}
    \resizebox{0.36\linewidth}{!}{\input{./fig/abstract_partition_tree7.tex}}
  }
  \caption{ (continuation) }
\end{center}
\end{figure}

\begin{table}
  \caption{\label{tab:incremental_partition} Centralized Incremental Partition Algorithm }
  \vspace{-3em}

\begin{quote}
{\small
      
{\rule[0em]{\linewidth}{1pt}}
INCREMENTAL\_PARTITION$(\EE, p_{\emptyset})$

\begin{algorithmic}[1]

\STATE \COMMENT{Compute and Insert Root Vertex into $\Tp$}
\STATE $c_\emptyset \leftarrow \VVver( p_\emptyset )$;
\FOR{each gap edge $g$ of $c_\emptyset$}
     \STATE label $g$ as {\tt unexplored} in $c_\emptyset$;
\ENDFOR
\STATE insert $(p_\emptyset, c_\emptyset)$ into $\Tp$;

\STATE \COMMENT{Main Loop}
\WHILE{ any cell in $\Tp$ has {\tt unexplored} gap edges }

     \STATE $c_\zeta \leftarrow$ any cell in $\Tp$ with {\tt unexplored} gap edges;
     \STATE $g \leftarrow$ any {\tt unexplored} gap edge of $c_\zeta$;
     \STATE $(p_\xi, c_\xi) \leftarrow$ CHILD$( \EE, \Tp, \zeta, g )$; \COMMENT{See Tab.~\ref{tab:incremental_partition_child}}

     \STATE \COMMENT{Check for Branch Conflicts}
     \IF{ there exists any cell $c_{\xi'}$ in $\Tp$ which is in \emph{branch conflict} with $c_{\xi}$ }
          \STATE discard $(p_\xi, c_\xi)$;
          \STATE label $g$ as {\tt phantom\_wall} in $c_\zeta$;
     \ELSE
          \STATE insert $( p_\xi, c_\xi )$ into $\Tp$;
          \STATE label $g$ as {\tt child} in $c_\zeta$;
     \ENDIF

\ENDWHILE

\STATE {\bf return} $\Tp$;

\end{algorithmic}
\vspace{-0.5em}

{\rule[0.3em]{\linewidth}{0.5pt}}

}
\end{quote}
\end{table}
\begin{table}
  \caption{\label{tab:incremental_partition_child} Incremental Partition Subroutine }
  \vspace{-3em}

\begin{quote}
{\small
      
{\rule[0em]{\linewidth}{1pt}}
CHILD$(\EE, \Tp, \zeta, g)$

\begin{algorithmic}[1]

\STATE $\xi \leftarrow \successor(\zeta,i)$, where $g$ is the $i$th nonparent gap edge of $c_\zeta$ counterclockwise from  $p_\zeta$;

\IF{ $|V_{c_\xi}| > 3$} 
     \STATE enumerate $c_\zeta$'s vertices $1,2,3,\ldots$ counterclockwise from $p_\zeta$;
\ELSE 
     \STATE enumerate $c_\zeta$'s vertices so that $p_\zeta$ is assigned $1$ and the remaining vertices of $c_\zeta$ are assigned $2$ and $3$ 
            \\ such that the vertex assigned $3$ is on the {\tt parent} gap edge of $c_\zeta$;
\ENDIF

\STATE $p_{\rm \xi} \leftarrow$ vertex on $g$ assigned an odd integer in the enumeration;
\STATE $c_\xi \leftarrow \VVver(p_\xi)$;
\STATE truncate $c_\xi$ at $g$ such that only the portion remains which is across $g$ from $p_\zeta$; 
\STATE delete from $c_\xi$ any vertices which lie across a phantom wall from $p_\xi$;

\FOR{ each gap edge $g'$ of $c_\xi$ }
     \IF{ $g' == g$ }
          \STATE label $g'$ as {\tt parent} in $c_\xi$;
     \ELSIF{ $g'$ coincides with an existing phantom wall }
          \STATE label $g'$ as {\tt phantom\_wall} in $c_\xi$;
     \ELSE
          \STATE label $g'$ as {\tt unexplored} in $c_\xi$;
     \ENDIF
\ENDFOR

\STATE {\bf return} $(p_\xi, c_\xi)$;

\end{algorithmic}
\vspace{-0.5em}

{\rule[0.3em]{\linewidth}{0.5pt}}

}
\end{quote}
\end{table}

\begin{figure}[t!]
\begin{center}
  \resizebox{0.6\linewidth}{!}{ \includegraphics{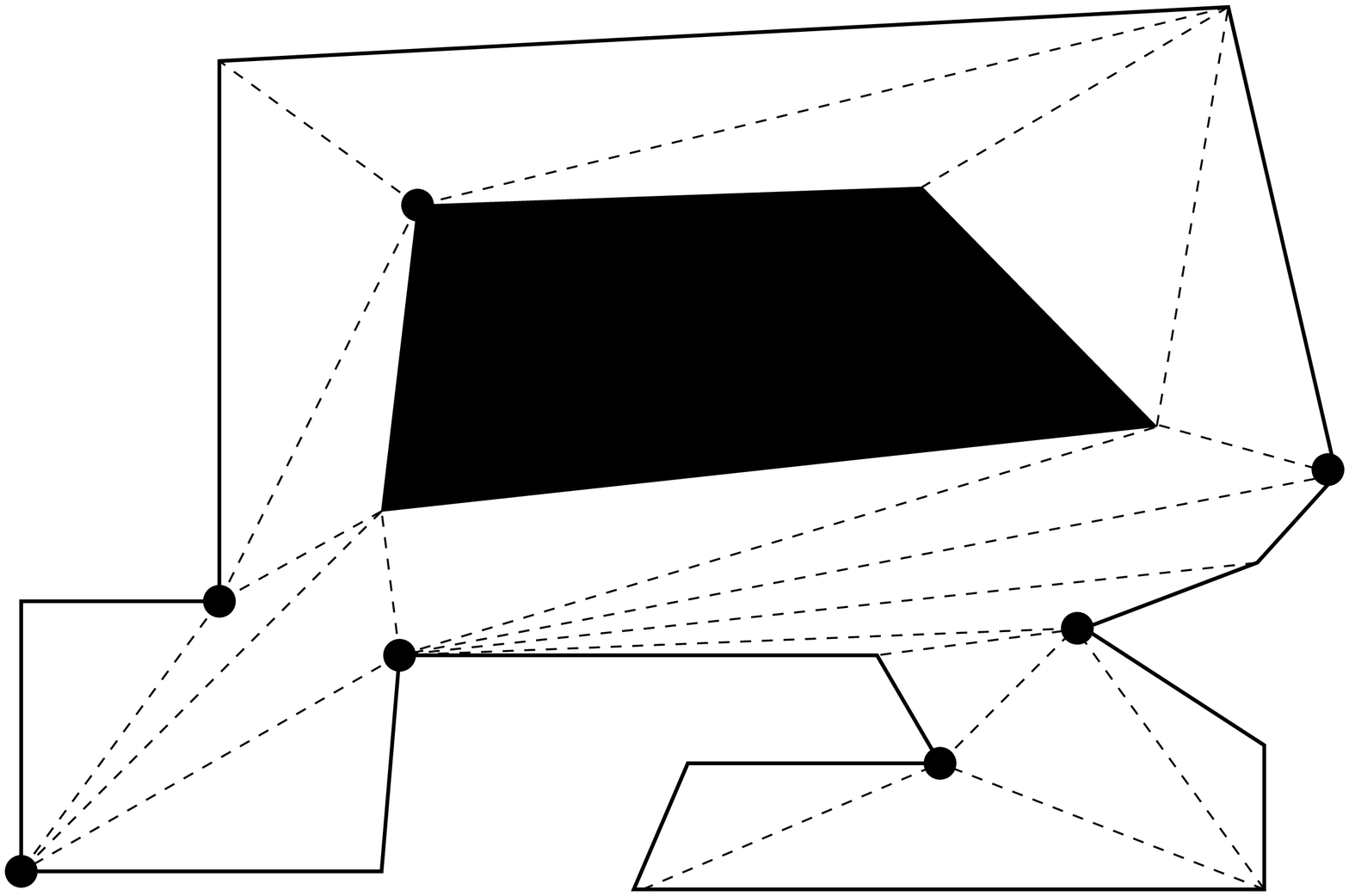}}
  \caption{\label{fig:triangulation} The partition tree produced by
    the centralized incremental partition algorithm of
    Table~\ref{tab:incremental_partition} or the distributed deployment
    algorithm of Table~\ref{tab:dfcd} can be used to triangulate an
    environment, as shown here for the simple example of
    Fig.~\ref{fig:incremental_partition}.  The triangulation is
    constructed by drawing diagonals (dashed lines) from each vantage
    point (black disks) to the visible environment vertices in its
    cell.}
\end{center}
\end{figure}

Now we turn our attention to the pseudocode
Table~\ref{tab:incremental_partition} for a precise description of the
algorithm.  The input is the environment $\EE$ and a single point
$p_\emptyset \in V_\EE$.  The output is the partition tree $\Tp$.  We
have seen that each vertex of the partition tree is a vantage point -
cell pair.  In particular, a cell is a data structure which stores not
only a polygonal boundary, but also a label on each of the polygon's
gap edges.  A gap edge label takes one of four possible values: {\tt
  parent}, {\tt child}, {\tt unexplored}, or {\tt phantom wall}.
These labels allow the following exact definition of the partition
tree.

\begin{definition}[Partition Tree $\Tp$]
\label{defn:partition_tree}
The directed rooted partition tree $\Tp$ has
\begin{enumerate}
\item vertex set consisting of vantage point - cell pairs produced by
  the incremental partition algorithm of
  Table~\ref{tab:incremental_partition}, and
\item a directed edge from vertex $(p_\zeta, c_\zeta)$ to vertex
  $(p_\xi, c_\xi)$ if and only if $c_\zeta$ has a {\tt child} gap edge
  which coincides with a {\tt parent} gap edge of $c_\xi$.
\end{enumerate}
\end{definition}

\noindent Stepping through the pseudocode
Table~\ref{tab:incremental_partition}, lines 1-5 compute and insert
the root vertex $(p_\emptyset, c_\emptyset)$ into $\Tp$.  Upon
entering the main loop at line 7, line 8 selects a cell $c_\zeta$
arbitrarily from the set of cells in $\Tp$ which have {\tt unexplored}
gap edges.  Line 9 selects an arbitrary {\tt unexplored} gap edge $g$
of $c_\zeta$.  The next vantage point candidate will be placed on an
endpoint of $g$ by a call on line 10 to the CHILD function of
Table~\ref{tab:incremental_partition_child}.  The PTVUID $\xi$ is
computed by the successor function on line 1 of
Table~\ref{tab:incremental_partition_child}.  For any $d$-tuple
$\zeta$ and positive integer $i$, $\successor(\zeta,i)$ is simply the
$(d+1)$-tuple which is the concatenation of $\zeta$ and $i$, e.g.,
$\successor((2,1),1)) = (2,1,1)$.  The CHILD function constructs a
candidate vantage point $p_\xi$ and cell $c_\xi$ as follows.  In the
typical case, when the parent cell $c_\zeta$ has more than three
edges, $c_\zeta$'s vertices are enumerated counterclockwise from
$p_\zeta$, e.g., as $c_\emptyset$'s vertices in
Fig.~\ref{fig:incremental_partition}a or Fig.~\ref{fig:special_case}.
In the special case of $c_\zeta$ being a triangle, e.g., as the
triangular cells in Fig.~\ref{fig:special_case}, $c_\zeta$'s vertices
are enumerated such that the $3$ lands on $c_\zeta$'s parent gap edge.
The vertex of $g$ which is odd in the enumeration is selected as
$p_\xi$.  Occasionally there may be \emph{double vantage points}
(colocated), e.g., as $p_{(2)}$ and $p_{(3)}$ in
Fig.~\ref{fig:special_case}.  We will see in
Sec.~\ref{subsec:sparse_vantage_point_set} that this
\emph{parity-based vantage point selection scheme} is important for
obtaining a special subset of the vantage points called the
\emph{sparse vantage point set}.  Returning to
Table~\ref{tab:incremental_partition}, the final portion of the main
loop, lines 11-17, checks whether $c_\xi$ is in \emph{branch conflict}
or $(p_\xi, c_\xi)$ should be added permanently to $\Tp$.  A cell
$c_\xi$ is in branch conflict with another cell $c_{\xi'}$ if and only
if $c_\xi$ and $c_{\xi'}$ are not openly disjoint (see
Fig.~\ref{fig:branch_conflict}).  The main algorithm terminates when
there are no more unexplored gap edges in $\Tp$.

\begin{figure}[t!]
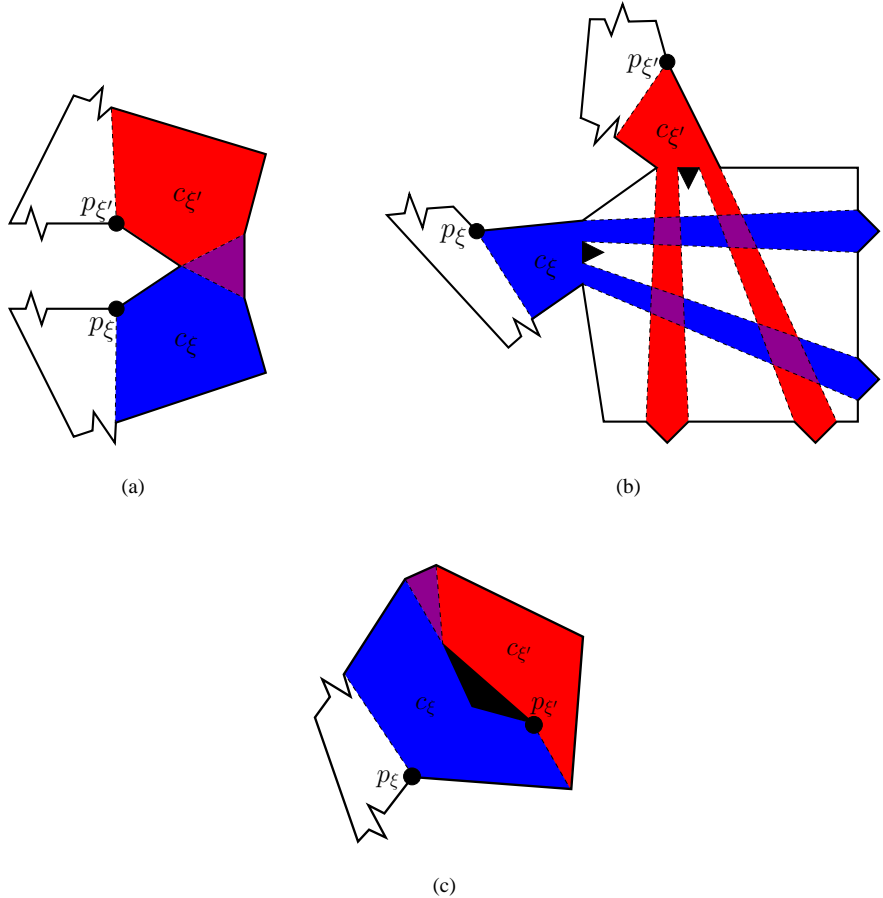

\begin{center}
  \subfigure[]{ \resizebox{0.21\linewidth}{!}{\input{./fig/deconflict1.tex}} } \hspace{3em}  
  \subfigure[]{ \resizebox{0.40\linewidth}{!}{\input{./fig/deconflict2.tex}} }
  \vspace{1.5em}

  \subfigure[]{ \resizebox{0.22\linewidth}{!}{\input{./fig/deconflict4.tex}} }
  \caption{\label{fig:branch_conflict} The incremental partition
    algorithm of Table~\ref{tab:incremental_partition} and distributed
    deployment algorithm of Table~\ref{tab:dfcd} may discard a cell
    $c_\xi$ if it is in \emph{branch conflict} with another cell
    $c_{\xi'}$ already in the partition tree, i.e., when $c_\xi$ and
    $c_{\xi'}$ and are not openly disjoint.  In these three examples,
    blue represents one cell $c_\xi$, red another cell $c_{\xi'}$, and
    purple their intersection $c_\xi \cap c_{\xi'}$.  A cell can even
    conflict with it's own parent if they enclose a hole as in (c).}
\end{center}
\end{figure}

An important difference between our incremental partition algorithm
and that of Ganguli et al \cite{AG:07} is that the set of cells
computed by our incremental partition is not unique.  This is because
the freedom in choosing cell $c_\zeta$ and gap $g$ on lines 8-9 of
Table~\ref{tab:incremental_partition} allows different executions of
the algorithm to fill the same part of the environment with different
branches of $\Tp$.  This may result in different sets of phantom walls
as well.  A phantom wall is only created on line 14 of
Table~\ref{tab:incremental_partition} when there is a branch conflict.
This discarding may seem computationally wasteful because the
environment could just be made simply connected by choosing $h$
phantom walls (one for each hole) prior to executing the algorithm.
Such an approach, however, would not be amenable to distributed
emulation without a priori knowledge of the environment.

\begin{figure}[t!]
\begin{center}
  \resizebox{0.85\linewidth}{!}{\input{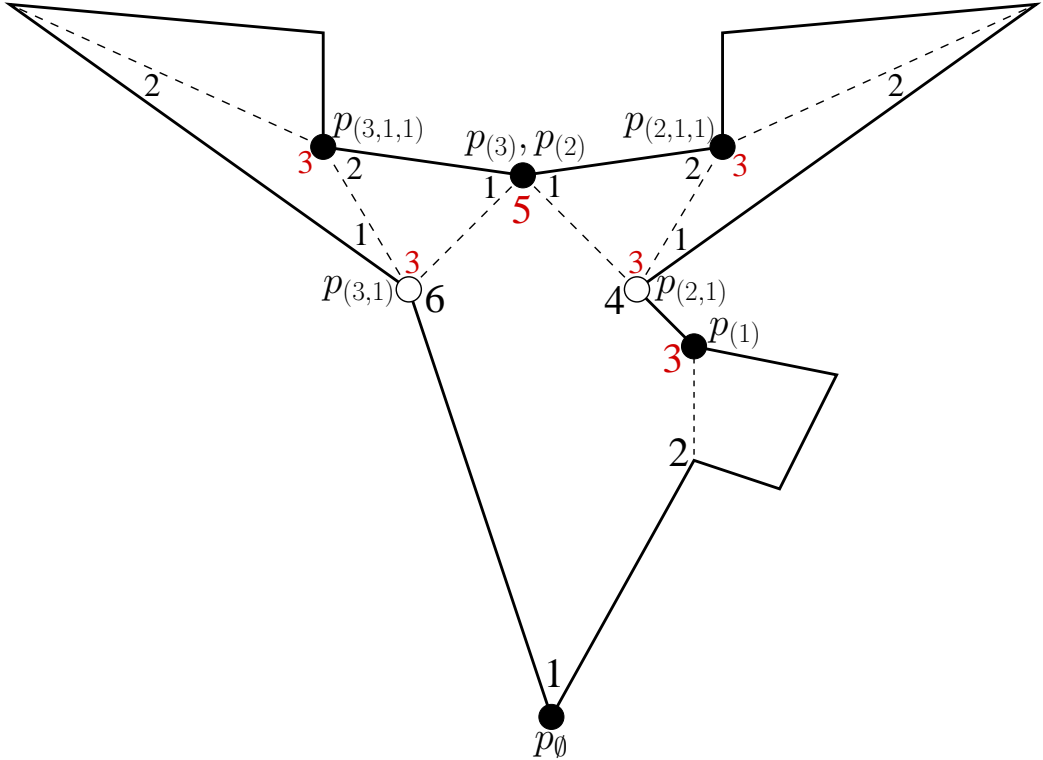}}
  \caption{\label{fig:special_case} The example used in
    Fig.~\ref{fig:incremental_partition} showed a typical incremental
    partition in which there were neither double vantage points nor
    any triangular cells.  This example, on the other hand, shows
    these special cases.  Disks, black or white, show vantage points
    produced by the incremental partition algorithm of
    Table~\ref{tab:incremental_partition}. Integers show enumerations
    of the cells used for the \emph{parity-based vantage point
      selection scheme}.  The double vantage points $p_{(2)}$ and
    $p_{(3)}$ are colocated.  The cells $c_{(2)}$, $c_{(3)}$,
    $c_{(2,1)}$, $c_{(3.1)}$, $c_{(2,1,1)}$, and $c_{(3,1,1)}$ are
    triangular.  The vantage points colored black are the \emph{sparse
      vantage points} found by the postprocessing algorithm of
    Table~\ref{tab:label_vantage_points}.  Under the distributed
    deployment algorithm of Table~\ref{tab:dfcd}, robotic agents
    position themselves at sparse vantage points.}
\end{center}
\end{figure}

The following important properties we prove for the incremental
partition algorithm are similar to properties we obtain for the
distributed deployment algorithm in
Sec.~\ref{sec:distributed_deployment}.

\begin{lemma}[Star-Convexity of Partition Cells]
\label{lm:star-convex_cells}
Any partition tree vertex $(p_\xi, c_\xi)$ constructed by the
incremental partition algorithm of
Table~\ref{tab:incremental_partition}, has the properties that
\begin{enumerate}
\item the cell $c_\xi$ is star-convex, and
\item the vantage point $p_\xi$ is in the kernel of $c_\xi$.
\end{enumerate}
\end{lemma}
\begin{proof}
  Given a star-convex set, say $S$, let $K$ be the kernel of $S$.
  Suppose that we obtain a new set $S'$ by truncating $S$ at a single
  line segment $l$ who's endpoints lie on the boundary $\partial
  S$. It is easy so see that the kernel of $S'$ contains $K \cap S'$,
  thus $S'$ must be star-convex if $K \cap S'$ is nonempty.  Indeed
  $l$ could not possibly block line of sight from any point in $K \cap
  S'$ to any point $p$ in $S'$, otherwise $p$ would have been
  truncated.  Inductively, we can obtain a set $S'$ by truncating the
  set $S$ at any finite number of line segments and the kernel of $S'$
  will be a superset of $S' \cap K$.  Now consider a partition tree
  vertex $(p_\xi, c_\xi)$.  By definition, the visibility polygon
  $\VV(p_\xi)$ is star-convex and $p_\xi$ is in the kernel.  By the
  above reasoning, the vertex-limited visibility polygon
  $\VVver(p_\xi)$ is also star-convex and has $p_\xi$ in its kernel
  because $\VVver(p_\xi)$ can be obtained from $\VV(p_\xi)$ by a
  finite number of line segment truncations (lines 8 and 9 of
  Table~\ref{tab:incremental_partition_child}).  Likewise, $c_\xi$ must
  be star-convex with $p_\xi$ in its kernel because $c_\xi$ is
  obtained from $\VVver(p_\xi)$ by a finite number of line segment
  truncations at the parent gap edge and phantom walls.
\end{proof}

\begin{theorem}[Properties of the Incremental Partition Algorithm]
  \label{thm:incremental_partition_convergence}
  Suppose the incremental partition algorithm of
  Table~\ref{tab:incremental_partition} is executed on an environment
  $\EE$ with $n$ vertices and $h$ holes.  Then
  \begin{enumerate} 
  \item the algorithm returns in finite time a partition tree $\Tp$
    such that every point in the environment is visible to some
    vantage point,
  \item the visibility graph of the vantage points $\Gvis{\EE}(\{
    p_\xi | (p_\xi, c_\xi) \in \Tp \})$ consists of a single connected
    component,
  \item the final number of vertices in $\Tp$ (and thus the total
    number of vantage points) is no greater than $n+2h-2$,
  \item there exist environments where the final number of vertices in
    $\Tp$ is equal to the upper bound $n+2h-2$, and
  \item the final number of phantom walls is precisely $h$.
  \end{enumerate}
\end{theorem}
\begin{proof}
  We prove the statements in order. The algorithm processes {\tt
    unexplored} gap edges one by one and terminates when there are no
  more {\tt unexplored} gap edges.  Once an {\tt unexplored} gap edge
  has been processed, it is never processed again because its label
  changes to {\tt phantom\_ wall} or {\tt child}.  Gap edges of cells
  are diagonals of the environment and there are no more than $
  \binom{n}{2}= \frac{n^2-n}{2}$ possible diagonals, which is finite,
  therefore the algorithm must terminate in finite time.
  Lemma~\ref{lm:star-convex_cells} guarantees that if the entire
  environment is covered by cells of $\Tp$, then every point is
  visible to some vantage point.  Suppose the final set of cells does
  not cover the entire environment.  Then there must be a portion of
  the environment which is topologically isolated from the rest of the
  environment by phantom walls, otherwise an {\tt unexplored} gap edge
  would have expanded into that region.  However, this would mean that
  a phantom wall was created at the {\tt parent} gap edge of a
  candidate cell which was not in branch conflict.  This is not
  possible because a phantom wall is only ever created if there is a
  branch conflict (lines 12-14 Table~\ref{tab:incremental_partition}).
  This completes the proof of statement (i).
  
  Statement (ii) follows from Lemma~\ref{lm:star-convex_cells}
  together with the fact that every vantage point is placed on the
  boundary of its parent's cell.  Given two vantage points in $\Tp$,
  say $p_\xi$ and $p_{\xi'}$, a path through $\Gvis{\EE}(\{ p_\xi |
  (p_\xi, c_\xi) \in \Tp \})$ from $p_\xi$ to $p_{\xi'}$ can be
  constructed as follows.  Follow parent-child visibility links up to
  the root vantage point $p_\emptyset$, then follow parent-child
  visibility links from $p_\emptyset$ down to $p_{\xi'}$.  Since such
  a path can always be constructed between any pair of vantage points,
  $\Gvis{\EE}(\{ p_\xi | (p_\xi, c_\xi) \in \Tp \})$ must consist of a
  single connected component.

  For statement (iii), we triangulate $\EE$ by triangulating the cells
  of $\Tp$ individually as in Fig.~\ref{fig:triangulation}.  Each cell
  $c_\xi$ is triangulated by drawing diagonals from $p_\xi$ to the
  vertices of $c_\xi$.  The total number of triangles in any
  triangulation of a polygonal environment with holes is $n+2h-2$
  (Lemma 5.2 in \cite{JOR:87}).  Since there is at least one triangle
  per cell and at most one vantage point per cell, the number of
  vantage points cannot exceed the maximum number of
  triangles $n+2h-2$.

  Statement (iv) is proven by the example in
  Fig.~\ref{fig:worst_cases}a.

  For statement (v), we argue topologically.  Suppose the final number
  of phantom walls were less than $h$.  Then somewhere two branches of
  the parition tree must share a gap edge with no phantom wall
  separating them.  If this shared gap edge is not a phantom wall, it
  must be either (1) a child in branch conflict, or (2) unexplored.
  Either way, the algorithm would have tried to create a cell there
  but then deleted it and created a phantom wall; a contradiction.
  Now suppose there were more than $h$ phantom walls.  Then a cell
  would be topologically isolated by phantom walls from the rest of
  the environment.  This is not possible because phantom walls can
  never be created at the parent-child gap edge between two cells.
  Since the final number of phantom walls can be neither less nor
  greater than $h$, it must be $h$.

  
\end{proof}

\begin{figure}[t!]
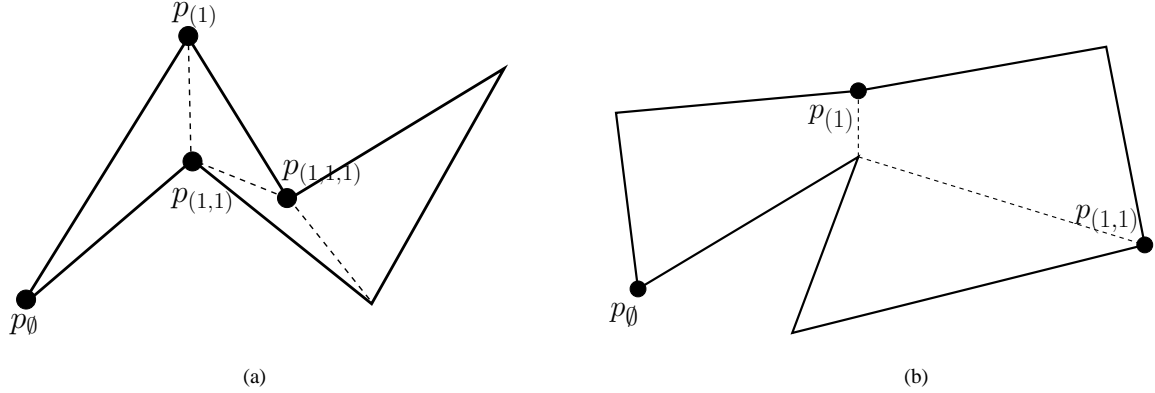

\begin{center}
  \subfigure[]{ \resizebox{0.4\linewidth}{!}{\input{./fig/incremental_partition_worst_case.tex}} } \hspace{2.5em}
  \subfigure[]{ \resizebox{0.5\linewidth}{!}{\input{./fig/sparse_vantage_point_worst_case.tex}} }
  \caption{\label{fig:worst_cases} 
(a) An example of when the final number of vantage points in $\Tp$ 
    is equal to the upper bound $n+2h-2$ given in
    Theorem~\ref{thm:incremental_partition_convergence}.  
(b) An example of when the number of points in $\real^2$ where at least one sparse 
    vantage point is located is equal to the upper bound 
    $\left \lfloor \frac{n+2h-1}{2} \right \rfloor$ given in
    Theorems~\ref{thm:sparse_vantage_point_bound} and
    \ref{thm:dfcd_convergence}.
  }
\end{center}
\end{figure}

\subsection{A Sparse Vantage Point Set}
\label{subsec:sparse_vantage_point_set}

Suppose we were to deploy robotic agents onto the vantage points
produced by the incremental partition algorithm (one agent per vantage
point).  Then, as Theorem~\ref{thm:incremental_partition_convergence}
guarantees, we would achieve our goal of complete visibility coverage
with connectivity.  The number of agents required would be no greater
than the number of vantage points, namely $n+2h-2$.  This upper bound,
however, can be greatly improved upon.  In order to reduce the number
of vantage points agents must deploy to, the postprocessing algorithm
in Table~\ref{tab:label_vantage_points} takes the partition tree output
by the incremental partition algorithm and labels a subset of the
vantage points called the \emph{sparse vantage point set}.  Starting
at the leaves of the partition tree and working towards the root,
vantage points are labeled either {\tt nonsparse} or {\tt sparse}
according to criterion on line 2 of
Table~\ref{tab:label_vantage_points}. As proven in
Theorem~\ref{thm:sparse_vantage_point_bound} below, the sparse vantage
points are suitable for the coverage task and their cardinality has a
much better upper bound than the full set of vantage points.  All the
vantage points in the example of Fig.~\ref{fig:incremental_partition}
are sparse.  Fig.~\ref{fig:special_case} shows an example of when only
a proper subset of the vantage points is sparse.

\begin{table}
  \caption{\label{tab:label_vantage_points} Postprocessing of Partition Tree }
  \vspace{-3em}

\begin{quote}
{\small
      
{\rule[0em]{\linewidth}{1pt}}
LABEL\_VANTAGE\_POINTS$(\EE, \Tp)$

\begin{algorithmic}[1]

\WHILE{ there exists a vantage point $p_\xi$ in $\Tp$ such that $p_\xi$ has not yet been labeled
        \\ \ \ \ \ \ \ \ \ \ {\bf and} $\bigl ($ $p_\xi$ is at a leaf {\bf or} all child vantage points of $p_\xi$ have been labeled $\bigr )$ }
     \IF{ $|V_{c_\xi}| == 3$ {\bf and} $p_\xi$ has exactly one child vantage point labeled {\tt sparse} }
          \STATE label $p_\xi$ as {\tt nonsparse};
     \ELSE
          \STATE label $p_\xi$ as {\tt sparse};
     \ENDIF
\ENDWHILE


\end{algorithmic}
\vspace{-0.5em}

{\rule[0.3em]{\linewidth}{0.5pt}}

}
\end{quote}
\end{table}

\begin{lemma}[Properties of a Child Vantage Point of a Triangular Cell]
\label{lm:child_vantage_point_of_triangular_cell}
Let $(p_\xi, c_\xi)$ be a partition tree vertex constructed by the
incremental partition algorithm of
Table~\ref{tab:incremental_partition} and suppose $c_\xi$ has a parent
cell $c_\zeta$ which is a triangle.  Then $p_\xi$ is in the kernel of
$p_\zeta$.  Furthermore, if $p_\zeta$ has a parent vantage point
$p_{\zeta'}$ (the grandparent of $p_\xi$), then $p_\xi$ is visible to
$p_{\zeta'}$.
\end{lemma}
\begin{proof}
  The kernel of a triangular (and thus convex) cell $c_\zeta$ is all
  of $c_\zeta$.  By Lemma~\ref{lm:star-convex_cells}, $p_{\zeta'}$ is
  in the kernel of $c_{\zeta'}$.  According to the parity-based
  vantage point selection scheme (line 5 of
  Table~\ref{tab:incremental_partition_child}), $p_\xi$ is located at a
  point common to $c_{\zeta'}$, $c_\zeta$, and $c_\xi$, therefore
  $p_\xi$ is in the kernel of $c_\zeta$ and visible to $c_{\zeta'}$.
\end{proof}

\begin{theorem}[Properties of the Sparse Vantage Point Set]
  \label{thm:sparse_vantage_point_bound}
  Suppose the incremental partition algorithm of
  Table~\ref{tab:incremental_partition} is executed to completion on an
  environment $\EE$ with $n$ vertices and $h$ holes and the vantage
  points of the resulting partition tree are labeled by the algorithm
  in Table~\ref{tab:label_vantage_points}.  Then
\begin{enumerate} 
  \item every point in the environment is visible to some sparse
        vantage point,
  \item the visibility graph of the sparse vantage points
        $\Gvis{\EE}(\{ p_\xi | (p_\xi, c_\xi) \in \Tp \})$ consists of
        a single connected component,
      \item the number of points in $\real^2$ where at least one
          sparse vantage point is located is no greater than $\left
          \lfloor \frac{n+2h-1}{2} \right \rfloor$, and
  \item there exist environments where the upper bound $\left
        \lfloor \frac{n+2h-1}{2} \right \rfloor$ in (iii) is met.
  \end{enumerate}
\end{theorem}
\begin{proof}

  Statements (i) and (ii) follow directly from
  Lemma~\ref{lm:child_vantage_point_of_triangular_cell} together with
  statements (i) and (ii) of
  Theorem~\ref{thm:incremental_partition_convergence}.
  
  For statement (iii) we use a triangulation argument similar to that
  used in \cite{AG:07} for environments without holes.  We use the
  same triangulation as in the proof of
  Theorem~\ref{thm:incremental_partition_convergence}
  (Fig.~\ref{fig:triangulation}).  The total number of triangles in
  any triangulation of a polygonal environment with holes is $n+2h-2$
  (Lemma 5.2 in \cite{JOR:87}).  Suppose we can assign at least one
  unique triangle to $p_\emptyset$ whenever $p_\emptyset$ is sparse
  and at least two unique triangles to all other sparse vantage point
  locations.  Let $N_{\rm sparse}$ be the number of sparse vantage
  point locations.  Setting $2 (N_{\rm sparse}-1) + 1 = 2N_{\rm
    sparse}-1$ to be less or equal to the total number of triangles
  $n+2h-2$ and solving for $N_{\rm sparse}$ gives the desired bound
  \[
  N_{\rm sparse} \leq
  \left \lfloor \frac{(n+2h-2)+1}{2} \right \rfloor = \left \lfloor
    \frac{n+2h-1}{2} \right \rfloor.
  \]
  Indeed we can make such an assignment of triangles to sparse vantage
  point locations.  Our argument relies on the parity-based vantage
  point selection scheme and the criterion for labeling a vantage
  point as {\tt sparse} on line 2 of
  Table~\ref{tab:label_vantage_points}.  To any sparse vantage point
  location, say of $p_\xi$ other than the root, we assign one triangle
  in the parent cell. The triangle in the parent cell is the triangle
  formed by its parent gap edge together with its parent's vantage
  point.  To each sparse vantage point location, say of $p_\xi$,
  including the root, we assign additionally one triangle in the cell
  $c_\xi$.  If $c_\xi$ has no children, then any triangle in $c_\xi$
  can be assigned to $p_\xi$.  If $c_\xi$ has children (in which case
  it must have greater than one triangle) we need to check that it has
  more triangles than child vantage point locations with odd parity.
  Suppose $c_\xi$ has an even number of edges.  Then this number of
  edges can be written $2m$ where $m \geq 2$.  The number of triangles
  in $c_\xi$ is $2m-2$ and the number of odd parity vertices in
  $c_\xi$ where child vantage points could be placed is $m-1$.  This
  means at most $m-1$ triangles in $c_\xi$ are assigned to odd parity
  child vantage point locations, which leaves $(2m-2)-(m-1) = m-1 \geq
  1$ triangles to be assigned to the location of $p_\xi$.  The case of
  $c_\xi$ having an odd number of edges is proven analogously.

  Statement (iv) is proven by the example in
  Fig.~\ref{fig:worst_cases}.

\end{proof}


\section{Distributed Deployment Algorithm}
\label{sec:distributed_deployment}

\begin{figure}[t!]
  \begin{center}
    \subfigure[]{ \resizebox{0.3\linewidth}{!}{ \includegraphics{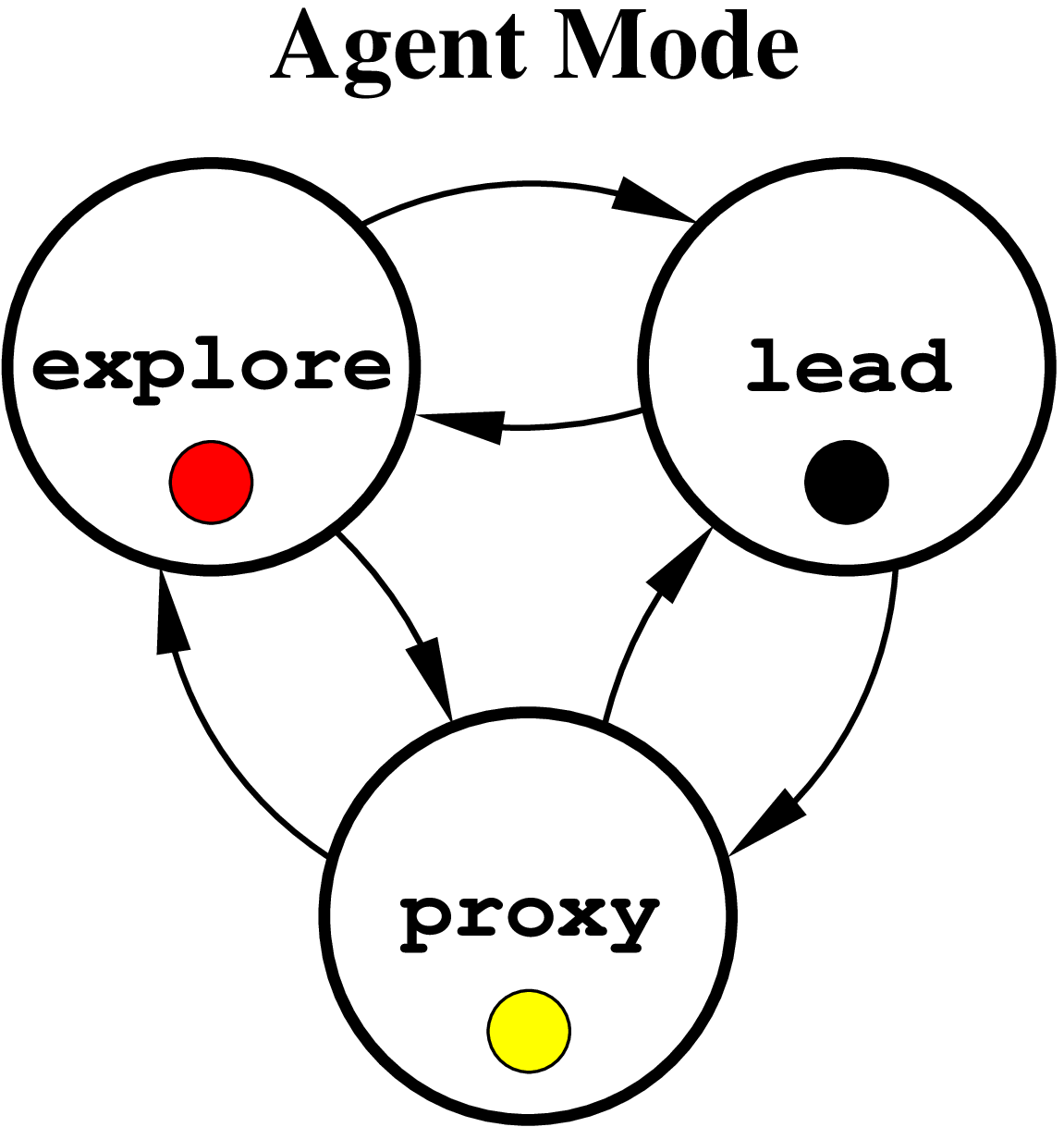} } } \hspace{3cm} 
    \subfigure[]{ \resizebox{0.4\linewidth}{!}{ \input{./fig/proxy_venn_diagram1.tex} } }
    \caption{\label{fig:agent_modes_and_abstract_proxy} (a) In the
      distributed deployment algorithm of Table~\ref{tab:dfcd}, each
      agent may switch between {\tt lead}, {\tt proxy}, and {\tt
        explore} mode based on certain asynchronous events.  Leader
      agents are responsible for maintaining a distributed
      representation of the partition tree $\Tp$, proxies help
      establish communication for solving branch conflicts, and
      explorers systematically navigate through $\Tp$ in search of
      opportunities to become a leader or
      proxy. 
      The agent mode color code is used also in
      Fig.~\ref{fig:df_ordering_micro} and \ref{fig:minimal_example}.
      (b) Even if a pair of leader agents (black) are not mutually
      visible, their cells ($c_\xi$ and $c_{\xi'}$) may intersect as
      in Fig.~\ref{fig:branch_conflict}, shown here abstractly by a
      Venn diagram.  Sending a proxy agent (yellow), on a \emph{proxy
        tour} around one of the cell boundaries guarantees it will
      enter the cells' intersection so that communication between
      leaders can be proxied.  The leaders can then establish a local
      common reference frame and compare cell boundaries in order to
      solve branch conflicts. }
  \end{center}
\end{figure}

\begin{figure}[t!]
\begin{center}
 \subfigure[]{ \resizebox{0.42\linewidth}{!}{ \includegraphics{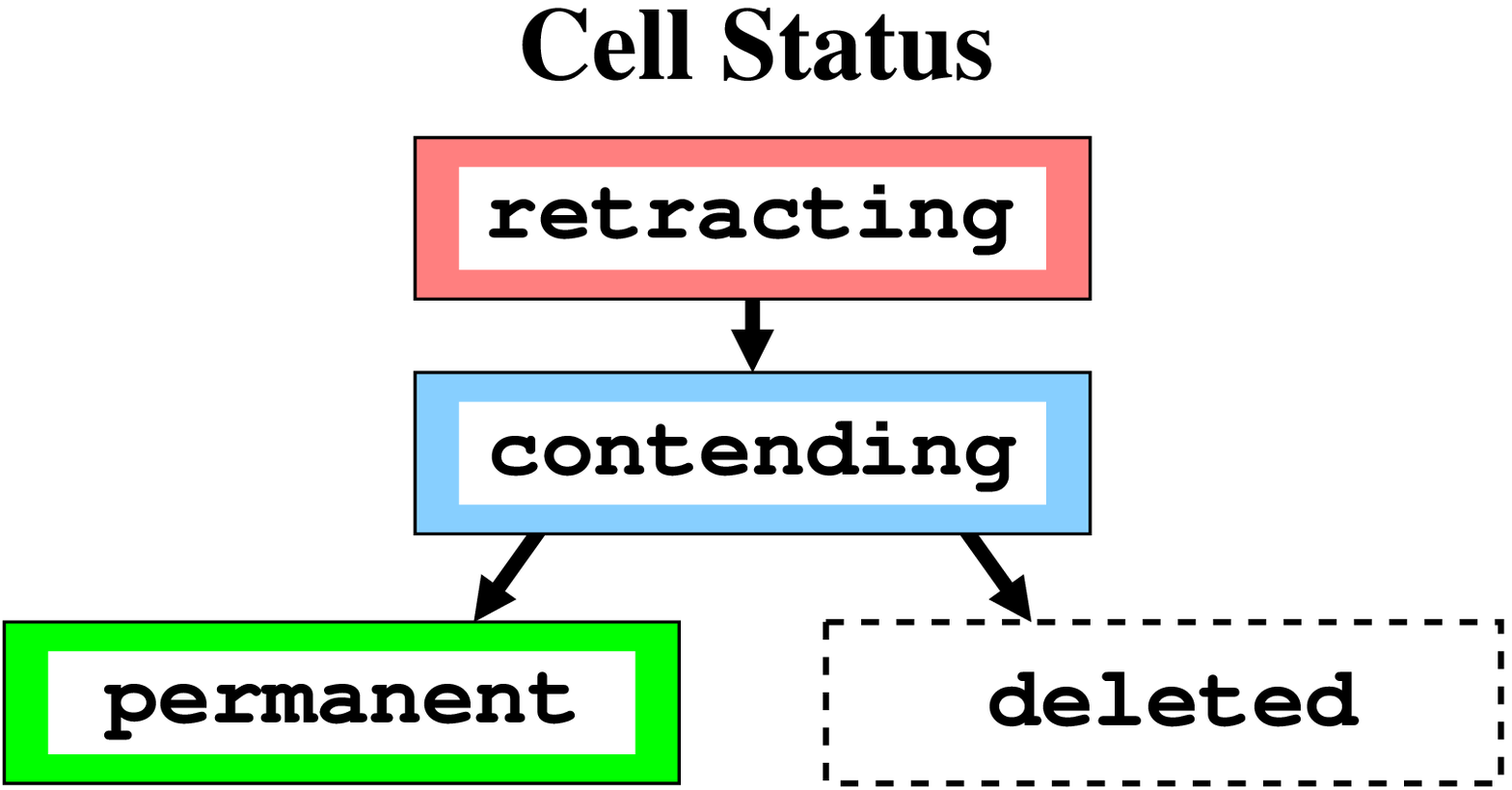} } }

 \subfigure[]{ \resizebox{0.21\linewidth}{!}{ \includegraphics{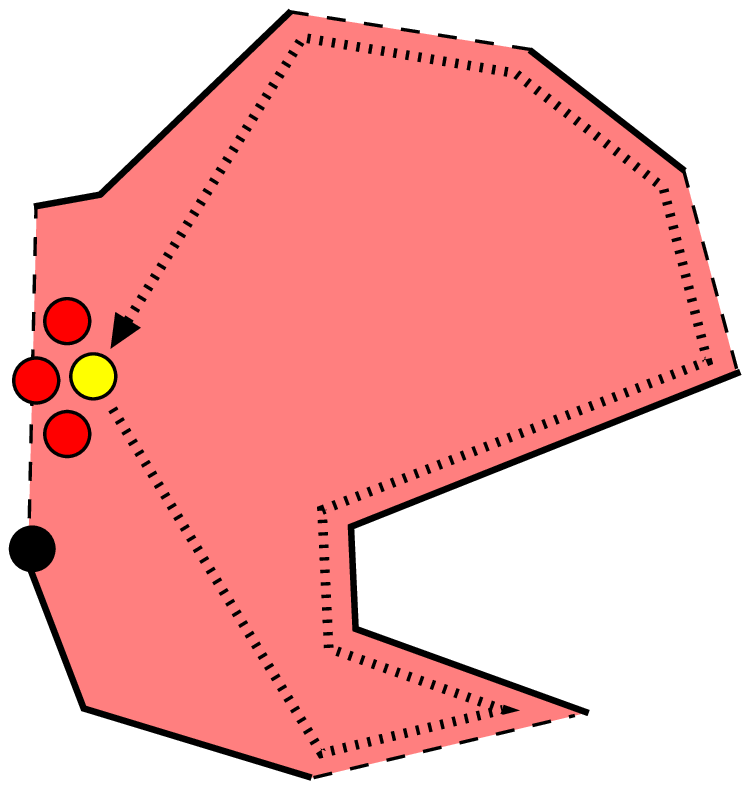} } }
 \subfigure[]{ \resizebox{0.21\linewidth}{!}{ \includegraphics{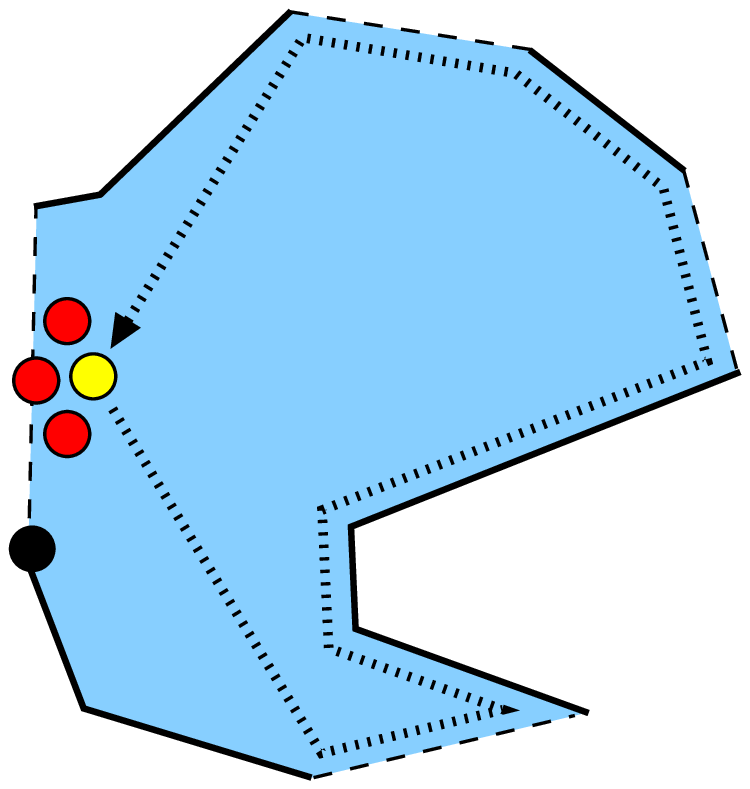} } }
 \subfigure[]{ \resizebox{0.21\linewidth}{!}{ \includegraphics{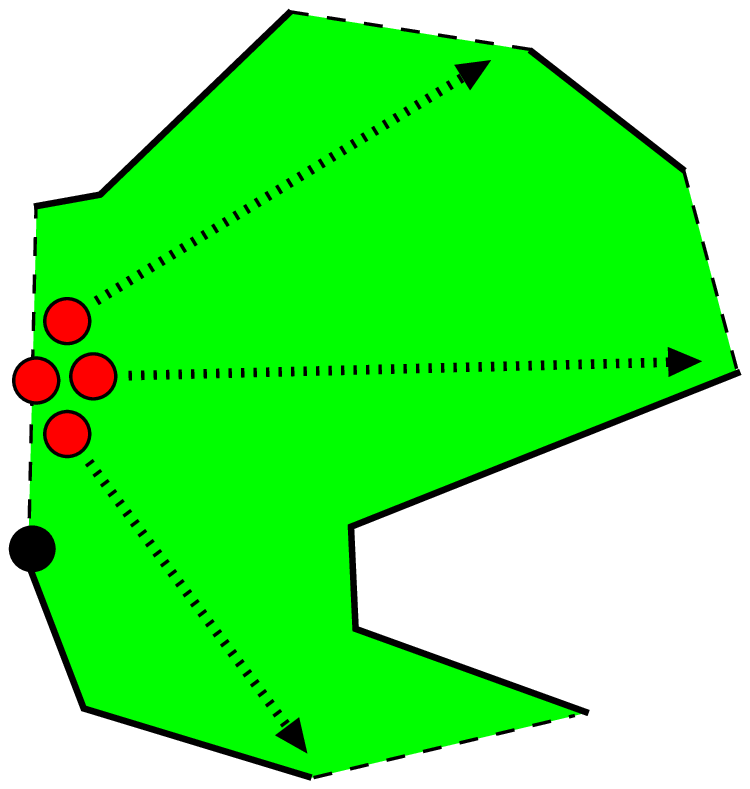} } }
 \caption{\label{fig:cell_statuses} (a) In the distributed deployment
   algorithm of Table~\ref{tab:dfcd}, any cell in a leader's memory
   has a status which takes the value {\tt retracting}, {\tt
     contending}, or {\tt permanent}.  (b) Each cell status is
   initially {\tt retracting}.  The status of a retracting cell is
   advanced to {\tt contending} after the execution of a proxy tour in
   which the cell is truncated as necessary to ensure no branch
   conflict with any permanent cells.  (c) In a second proxy tour, a
   contending cell is deleted if it is found to be in branch conflict
   with another contending cell of smaller PTVUID (according to total
   ordering Def.~\ref{defn:ptvuid_total_ordering}), otherwise its
   status is advanced to {\tt permanent}.  (d) Only when a cell has
   attained status {\tt permanent} can any child cells be added at its
   unexplored gap edges (continued in
   Fig.~\ref{fig:df_ordering_micro}).  The cell status color code is
   used in Fig.~\ref{fig:df_ordering_micro} as well as
   \ref{fig:minimal_example}.}
\end{center}
\end{figure}

\begin{figure}[t!]
\begin{center}
  %
  %
  %
  \subfigure[]{ \resizebox{0.21\linewidth}{!}{ \includegraphics{./fig/cell_life_cycle4.eps} } }
  \subfigure[]{ \resizebox{0.21\linewidth}{!}{ \includegraphics{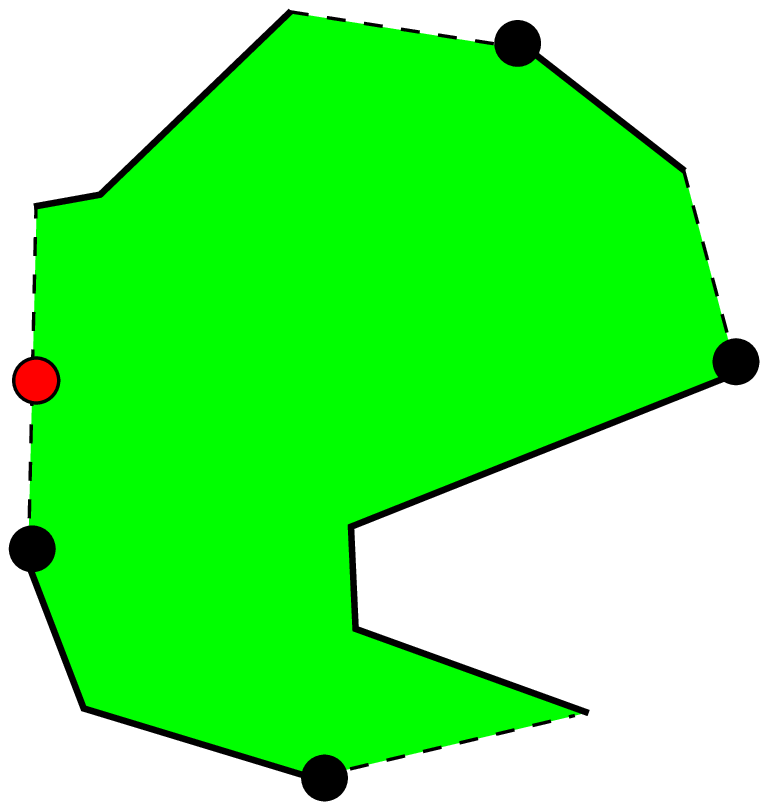} } }
  \subfigure[]{ \resizebox{0.21\linewidth}{!}{ \includegraphics{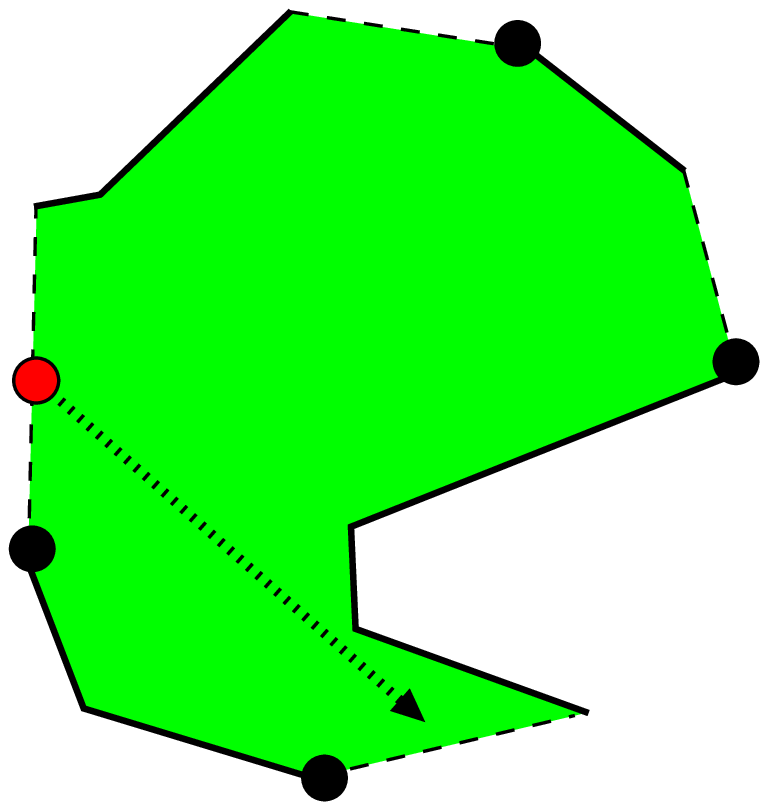} } }

  \subfigure[]{ \resizebox{0.21\linewidth}{!}{ \includegraphics{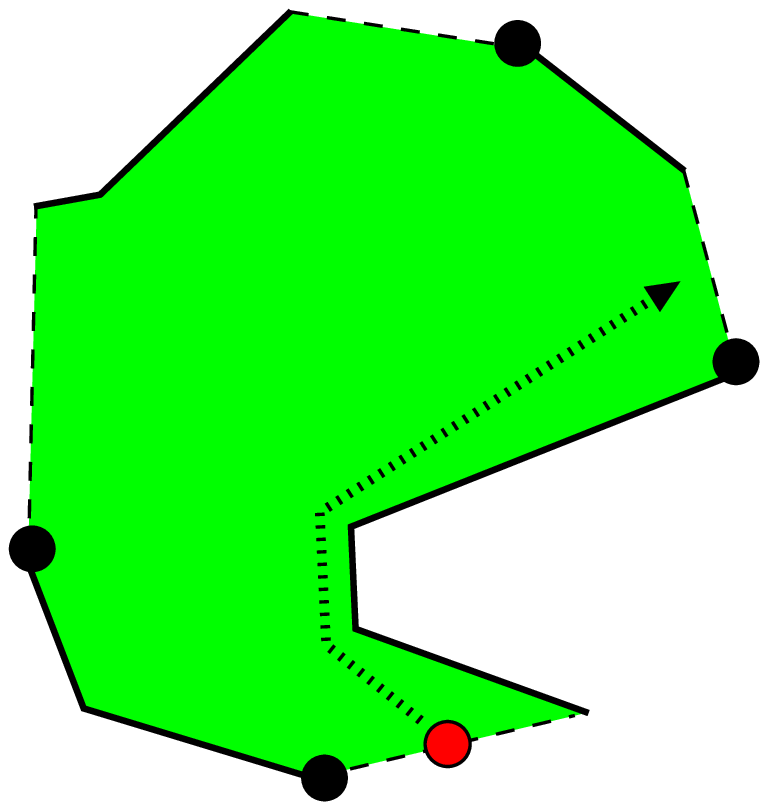} } }
  \subfigure[]{ \resizebox{0.21\linewidth}{!}{ \includegraphics{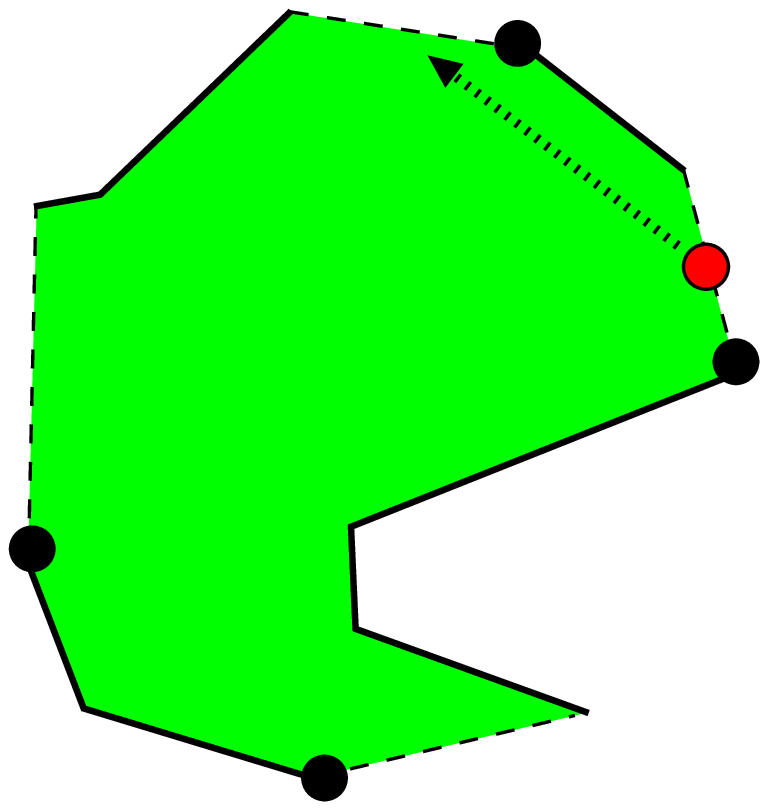} } }
  \subfigure[]{ \resizebox{0.21\linewidth}{!}{ \includegraphics{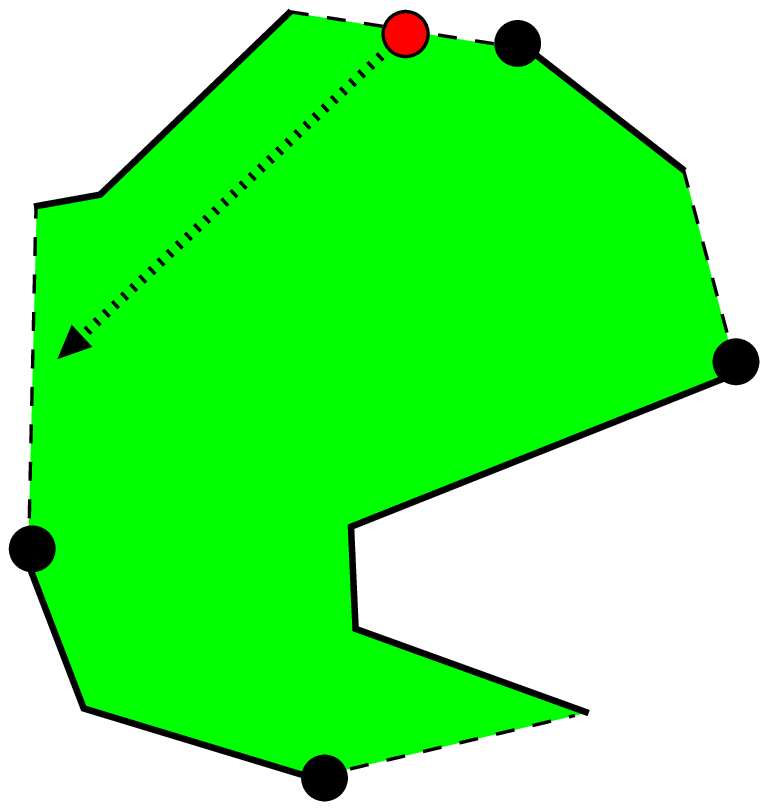} } }

  \caption{\label{fig:df_ordering_micro} Color codes correspond to
    those in Fig.~\ref{fig:agent_modes_and_abstract_proxy} and
    \ref{fig:cell_statuses}.  (a,b) Once a cell has status {\tt
      permanent}, arriving explorer agents can be sent to become
    leaders at child gap edges. (c-f) Any remaining explorer agents
    continue systematically navigating the partition tree in search of
    leader or proxy tasks they could perform .}
\end{center}
\end{figure}

\begin{figure}[t!]
  \begin{center}
    \resizebox{0.36\linewidth}{!}{ \includegraphics{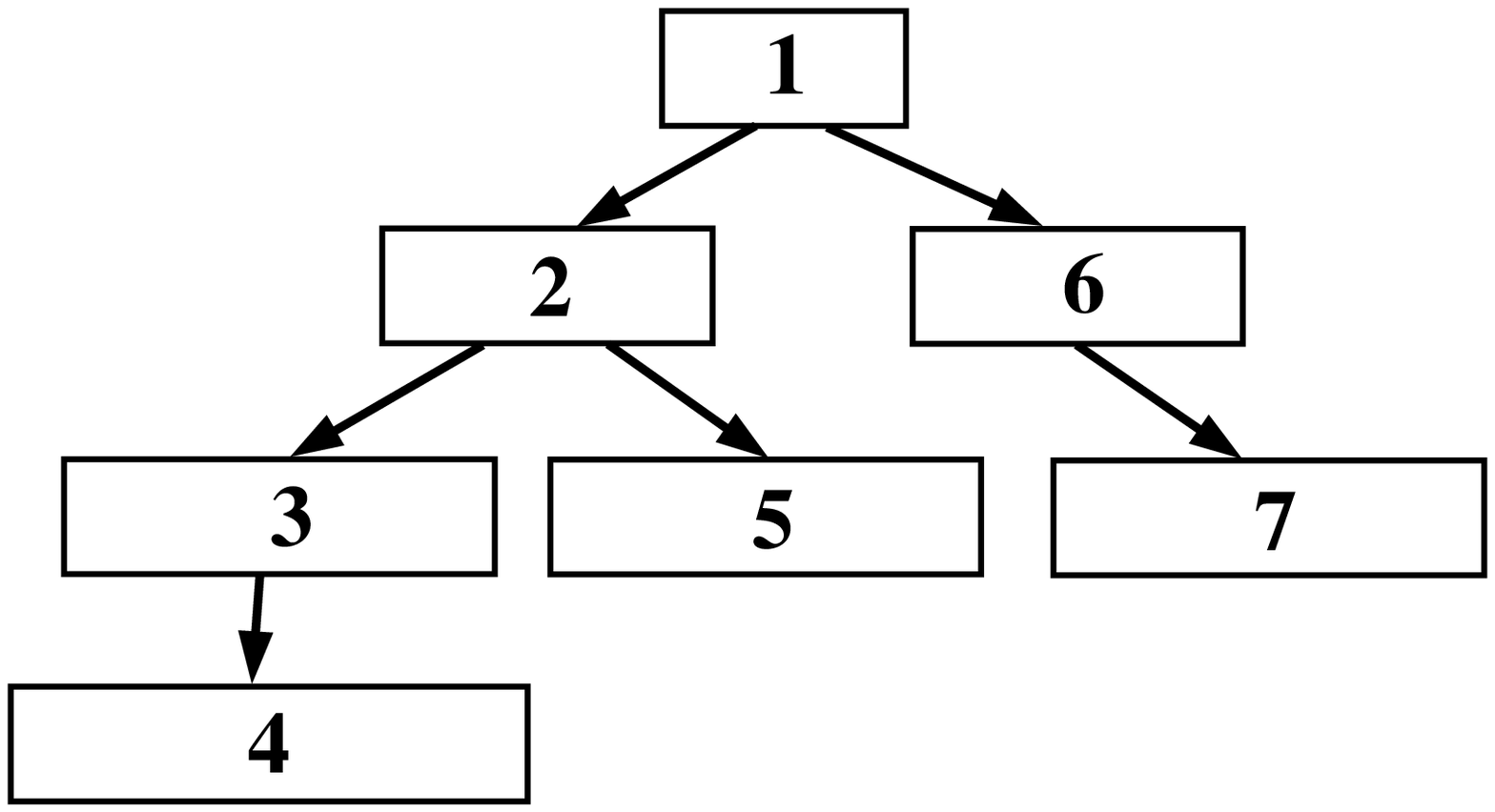} } \hspace{1cm} 
    \resizebox{0.36\linewidth}{!}{ \includegraphics{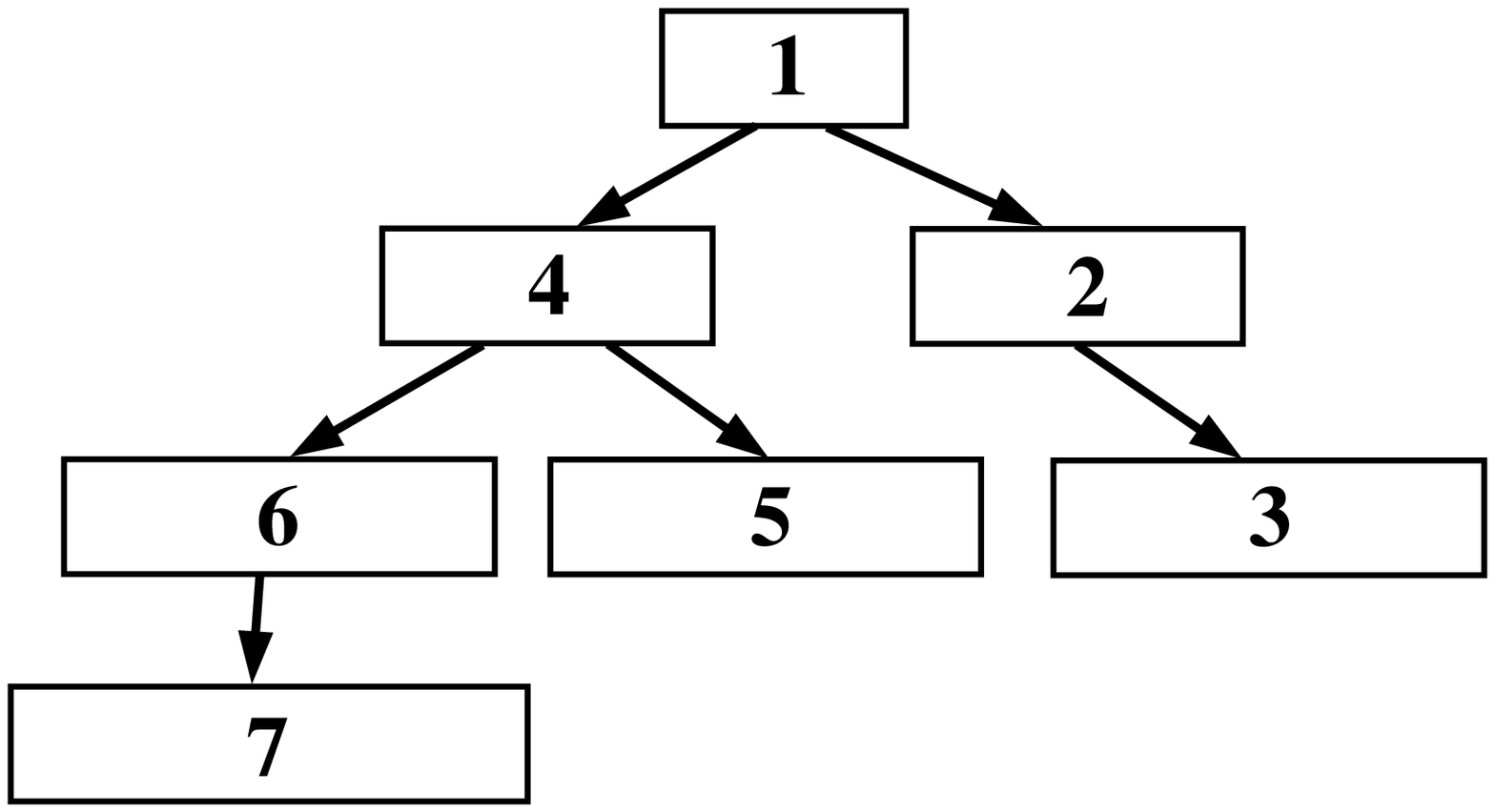} }
    \vspace{-6.0em}
    
    \caption{\label{fig:df_ordering_macro} In the distributed
      deployment algorithm of Table~\ref{tab:dfcd}, explorer agents
      search the partition tree $\Tp$ depth-first for leader or proxy
      tasks they could perform.  An agent in a cell, say $c_\xi$, can
      always order the gap edges of $c_\xi$, e.g., counterclockwise
      from the parent gap edge.  The depth-first search progresses by
      the explorer agent always moving to the next unvisited child or
      unexplored gap edge in that ordering.  The agent thus moves from
      cell to cell deeper and deeper until a leaf (a vertex with no
      children) is found.  Once at a leaf, the agent backtracks to the
      most recent vertex with unvisited child or unexplored gap edges
      and the process continues.  As an example, (left) integers (not
      to be confused with PTVUIDs) shows the depth-first order an
      agent would visit the vertices of $\Tp$ in
      Fig.~\ref{fig:incremental_partition}f if the gap edges in each
      cell were ordered couterclockwise from the parent gap edge.  If
      the agent instead uses a gap edge ordering cyclically shifted by
      one, then (right) shows the different resulting depth-first
      order.  If each agent uses a different gap edge ordering, e.g.,
      cyclically shifted by their UID, then different branches of
      $\Tp$ are explored in parallel and the deployment tends to cover
      the environment more quickly.
      Cf. Fig.~\ref{fig:df_ordering_micro}.}
  \end{center}
\end{figure}

\begin{figure}[t!]
\begin{center}
  \subfigure[]{ \resizebox{0.25\linewidth}{!}{ \includegraphics{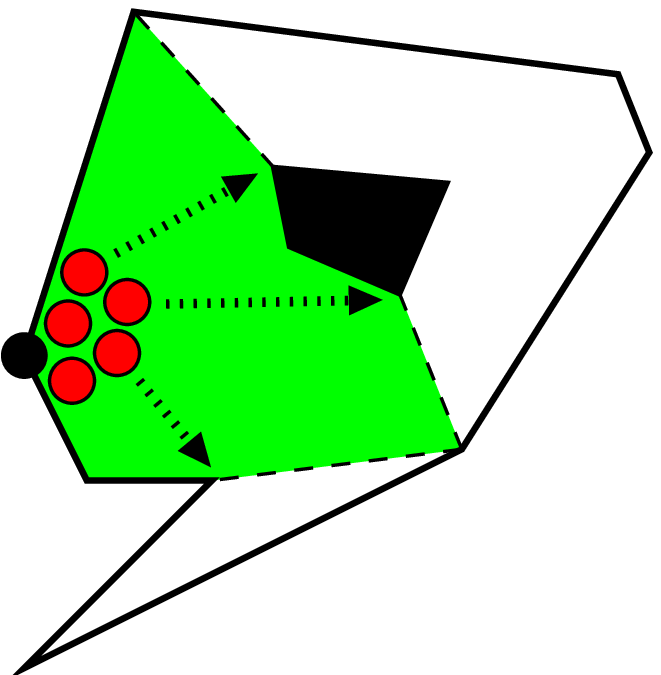} } }
  \subfigure[]{ \resizebox{0.25\linewidth}{!}{ \includegraphics{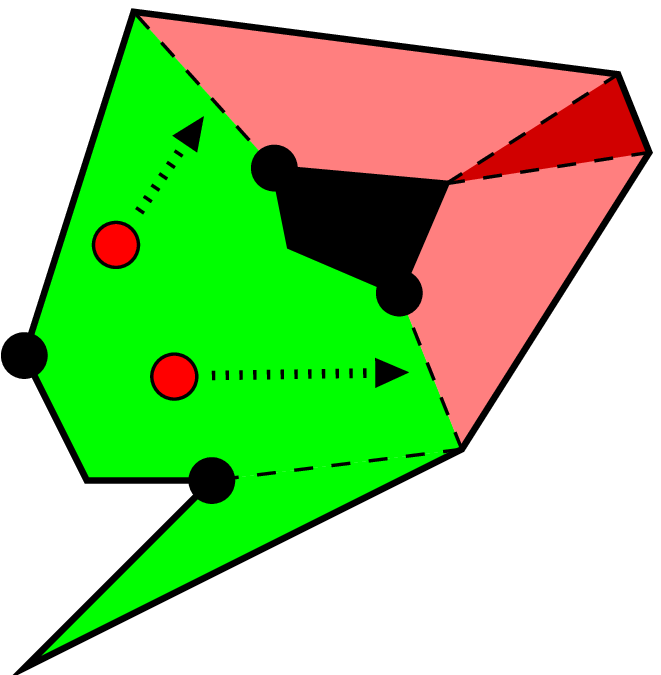} } }
  \subfigure[]{ \resizebox{0.25\linewidth}{!}{ \includegraphics{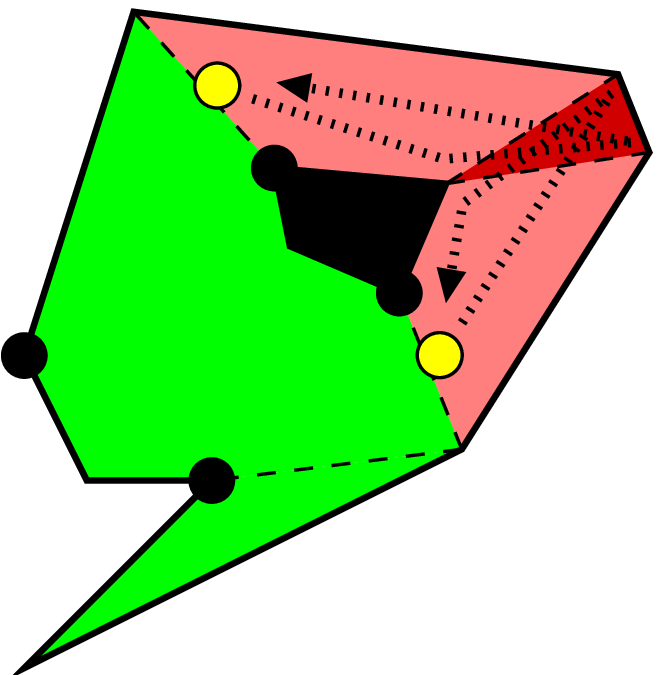} } }
  \vspace{0.5em}

  \subfigure[]{ \resizebox{0.25\linewidth}{!}{ \includegraphics{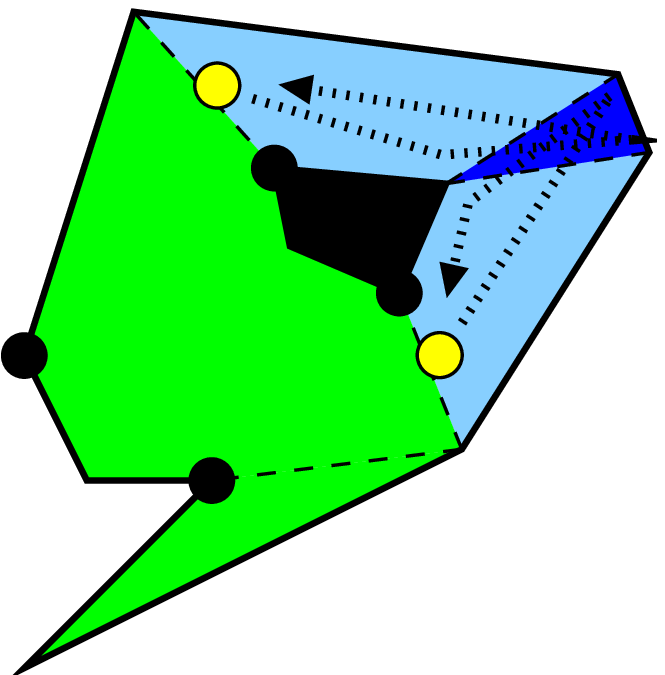} } }
  \subfigure[]{ \resizebox{0.25\linewidth}{!}{ \includegraphics{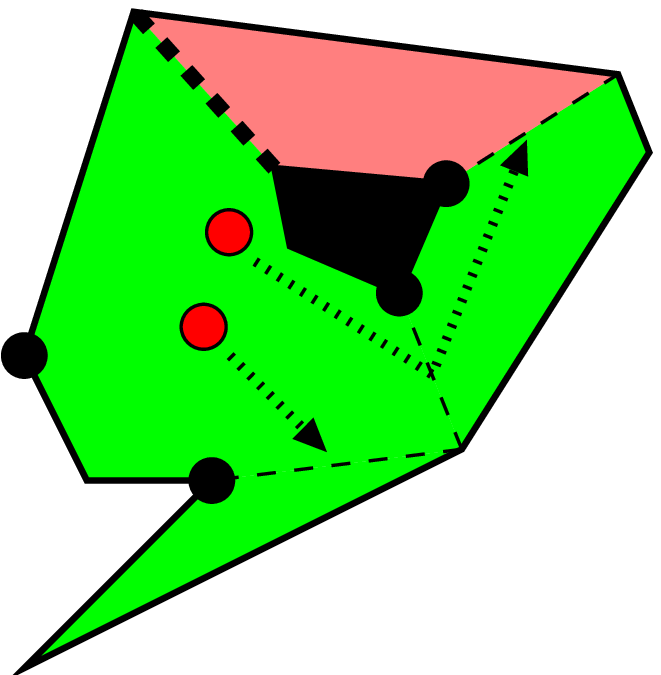} } }
  \subfigure[]{ \resizebox{0.25\linewidth}{!}{ \includegraphics{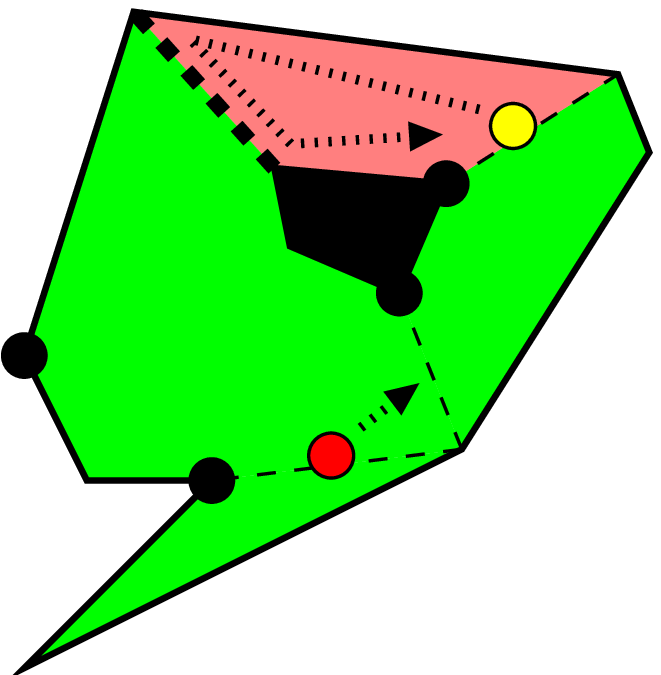} } }
  \vspace{0.5em}

  \subfigure[]{ \resizebox{0.25\linewidth}{!}{ \includegraphics{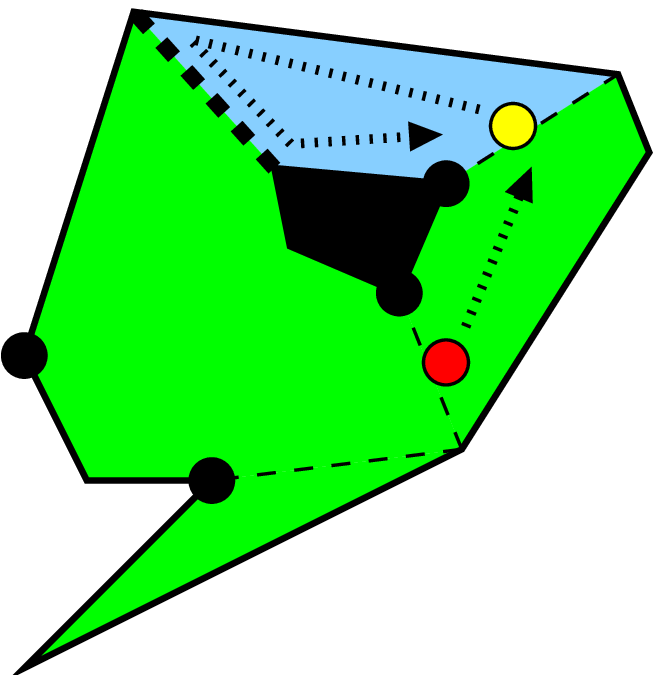} } }
  \subfigure[]{ \resizebox{0.25\linewidth}{!}{ \includegraphics{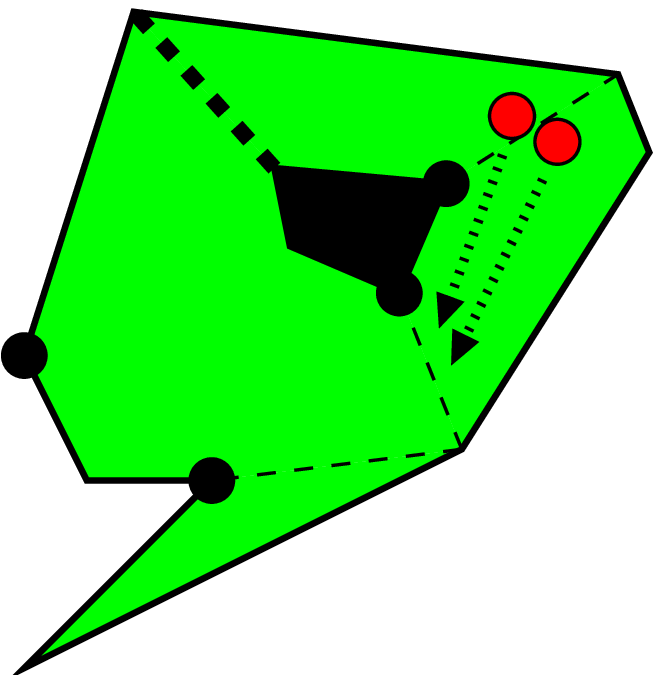} } }
  \caption{\label{fig:minimal_example} With color codes from
    Fig.~\ref{fig:agent_modes_and_abstract_proxy} and
    \ref{fig:cell_statuses}, here is a simple example of agents
    executing the distributed deployment algorithm of
    Table~\ref{tab:dfcd}.  (a) Agents enter the environment and the
    leader initializes the root cell to status {\tt permanent} because
    no branch conflicts could possibly exist yet.  Explorer agents
    move out to become leaders of child cells.  (b) The lower child
    cell is initialized with status {\tt permanent} because it has no
    gap edges and thus cannot be in branch conflict.  The upper two
    child cells are initialized to {\tt retracting} because they could
    be in branch conflict at unexplored gap edges; indeed there is a
    branch conflict at the dark red overlap region.  The remaining
    explorer agents continue moving out to the new cells.  (c) Once
    the explorers reach the retracting cells, they become proxies and
    run tours around the cells to check for branch conflict with
    permanent cells.  (d) After the first proxy tours, the child
    cells' statuses are advanced to {\tt contending} and each proxy
    run a second tour.  (e) During the second proxy tours, the branch
    conflict is detected between contending cells and the cell with
    higher PTVUID is deleted.  The agents that were in the deleted
    cell move back up the partition tree and continue exploring
    depth-first.  The other proxy becomes a leader of a new child cell
    initialized to {\tt retracting}.  (f) One of the explorers arrives
    at the retracting cell and begins a proxy tour to advance the cell
    to {\tt contending}.  (g) The proxy runs a second tour and
    advances the cell to {\tt permanent} and the partition is
    completed.  (h) Remaining explorers continue navigating the
    partition tree depth-first in search of tasks; this adds
    robustness because they will be able to fill in anywhere an agent
    may fail or a door may open.}
\end{center}
\end{figure}

In this section we describe how a group of mobile robotic agents can
distributedly emulate the incremental partition and vantage point
labeling algorithms of Sec.~\ref{sec:incremental_partition}, thus
solving the Distributed Visibility-Based Deployment Problem with
Connectivity.  We first give a rough overview of the algorithm, called
DISTRIBUTED\_DEPLOYMENT(), and later address details with aid of the
pseudocode in Table~\ref{tab:dfcd}.  Each agent $i$ has a local
variable mode$^{[i]}$, among others, which takes a value {\tt lead},
{\tt proxy}, or {\tt explore}.  For short, we call an agent in {\tt
  lead} mode a \emph{leader}, an agent in {\tt proxy} mode a
\emph{proxy}, and an agent in {\tt explore} mode an \emph{explorer}.
Agents may switch between modes (see
Fig.~\ref{fig:agent_modes_and_abstract_proxy}a) based on certain
asynchronous events.  Leaders settle at sparse vantage points and are
responsible for maintaining in their memory a distributed
representation of the partition tree $\Tp$ consistent with
Definition~\ref{defn:partition_tree}.  By distributed representation
we mean that each leader $i$ retains in its memory up to two
\emph{vertices of responsibility}, $(p_1^{[i]}, c_1^{[i]})$ and
$(p_2^{[i]}, c_2^{[i]})$, and it knows which gap edges of those
vertices lead to the parent and child vertices in $\Tp$.\footnote{The
  subscripts of a leader agent's \emph{vertices of responsibility} are
  not to be confused with PTVUIDs, i.e., $(p_1^{[i]}, c_1^{[i]})$ and
  $(p_2^{[i]}, c_2^{[i]})$ are not in general the same as $(p_{(1)},
  c_{(1)})$ and $(p_{(2)}, c_{(2)})$.}  We call $(p_1^{[i]},
c_1^{[i]})$ the \emph{primary vertex} of agent $i$ and $(p_2^{[i]},
c_2^{[i]})$ the \emph{secondary vertex}.  A leader typically has only
a primary vertex in its memory and may have also a secondary only if
it is either positioned (1) at a double vantage point, or (2) at a
sparse vantage point adjacent to a nonsparse vantage point.
%
%
Each cell in a leader's memory has a status which takes the value {\tt
  retracting}, {\tt contending}, or {\tt permanent} (see
Fig.~\ref{fig:cell_statuses}).  Only when a cell has attained status
{\tt permanent} can any child $\Tp$ vertices be added at its
unexplored gap edges.  

\begin{remark}[3 Cell Statuses]
\label{rm:3_cell_statuses}
In our system of three cell statuses, a cell must go through two steps
before attaining status {\tt permanent}.  Intuitively, the need for
two steps arises from the fact that an agent must first determine the
boundary of its cell before it can even know what other cells are in
branch conflict or place children according to the parity-based
vantage point selection scheme.  Hence, the first proxy tour allows
truncation of the cell boundary at all permanent cells.  Only after
that, when the boundary is known, is the second proxy tour run and the
cell deconflicted with other contending cells.  Note that even in the
centralized incremental partition algorithm two steps had to be taken
by a newly constructed cell: the cell had to be (1) truncated at
existing phantom walls, and then (2) deleted if it was in branch
conflict.\footnote{We did attempt to simplify the distributed
  deployment alogrithm and make the cells only go through a single
  step, i.e., a single proxy tour to become permanent, however, there
  seem to be other difficulties with such an approach, particularly
  with time complexity bounds.}
\end{remark}

\noindent The job of a proxy agent is to assist leaders in advancing
the status of their cells towards {\tt permanent} by proxying
communication with other leaders (see
Fig~\ref{fig:agent_modes_and_abstract_proxy}b).  Any agent which is
not a leader or proxy is an explorer.  Explorers merely move in
depth-first order systematically about $\Tp$ in search of opportunity
to serve as a proxy or leader (see Fig.~\ref{fig:df_ordering_micro}
and \ref{fig:df_ordering_macro}).
To simplify the presentation, let us assume for now that, as in the
examples Fig.~\ref{fig:incremental_partition} and
Fig.~\ref{fig:minimal_example}, no double vantage points or triangular
cells occur.  Under this assumption, each leader will be responsible
for only one $\Tp$ vertex, its primary vertex, and all vantage points
will be sparse.  The deployment begins with all agents colocated at
the first vantage point $p_\emptyset$.  One agent, say agent $0$, is
initialized to {\tt lead} mode with the first cell
$c_{\xi_1}^{[0]}=c_\emptyset=\VVver(p_\emptyset)$ in its memory.  All
other agents are initialized to {\tt explore} mode.  Agent $0$ can
immediately advance the status of $c_\emptyset$ to {\tt permanent}
because it cannot possibly be in branch conflict (no other cells even
exist yet); in general, however, cells can only transition between
statuses when a proxy tour is executed.  Agent $0$ sees all the
explorers in its cell and assigns as many as necessary to become
leaders so that there will be one new leader positioned on each
unexplored gap edge of $c_\emptyset$.  The new leader agents move
concurrently to their new respective vantage points while all
remaining explorer agents move towards the next cell in their
depth-first ordering.  When a leader first arrives at its vantage
point, say $p_\xi$, of the cell $c_\xi$, it initializes $c_\xi$ to
have status {\tt retracting} and boundary equal to the portion of
$\VVver(p_\xi)$ which is across the parent gap edge and extends away
from the parent's cell.  When an explorer agent comes to such a newly
created retracting cell, the leader assigns that explorer to become a
proxy and follow a proxy tour which traverses all the gap edges of
$c_\xi$.  During the proxy tour, the proxy agent is able to
communicate with any leader of a permanent cell that might be in
branch conflict with the $c_\xi$.  The cell $c_\xi$ is thus truncated
as necessary to ensure it is not in branch conflict with any {\tt
  permanent} cell.  When this first proxy tour is complete, the status
of $c_\xi$ is advanced to {\tt contending}.  The leader of $c_\xi$
then assigns a second proxy tour which again traverses all the gap
edges of $c_\xi$.  During this second proxy tour, the leader
communicates, via proxy, with all leaders of contending cells which
come into line of sight of the proxy.  If a branch conflict is
detected between $c_\xi$ and another contending cell, the agents have
a \emph{shoot-out}: they compare PTVUIDs of the cells and agree to
delete the one which is larger according to the following total
ordering.

\begin{definition}[PTVUID Total Ordering]
\label{defn:ptvuid_total_ordering}
Let $\xi_1$ and $\xi_2$ be distinct PTVUIDs.  If $\xi_1$ and $\xi_2$
do not have equal depth, then $\xi_1 < \xi_2$ if and only if the depth
of $\xi_1$ is less than the depth of $\xi_2$.  If $\xi_1$ and $\xi_2$
do have equal depth, then $\xi_1 < \xi_2$ if and only if $\xi_1$ is
lexicographically smaller than $\xi_2$.\footnote{ For example, $(1) <
  (2)$ and $(1,3) < (3,2)$, but $(3,2) < (1,3,1)$.}
\end{definition}


\noindent When a cell $c_\xi$ with parent $c_\zeta$ is deleted, two
things happen: (1) The leader of $c_\zeta$ marks a phantom wall at its
child gap edge leading to $c_\xi$, and (2) all agents that were in
$c_\xi$ become explorers, move back into $c_\zeta$, and resume
depth-first searching for new tasks as in
Fig.~\ref{fig:minimal_example}e.  If the second proxy tour of a cell
$c_\xi$ is completed without $c_\xi$ being deleted, then the status of
$c_\xi$ is advanced to {\tt permanent} and its leader may then assign
explorers to become leaders of child $\Tp$ vertices at $c_\xi$'s
unexplored gap edges.  Agents in different branches of $\Tp$ create
new cells in parallel and run proxy tours in an effort to advance
those cells to status {\tt permanent}.  New $\Tp$ vertices can in turn
be created at the unexplored gap edges of the new permanent cells and
the process continues until, provided there are enough agents, the
entire environment is covered and the deployment is complete.

%
%

We now turn our attention to pseudocode Table~\ref{tab:dfcd} to
describe DISTRIBUTED\_DEPLOYMENT() more precisely.  The algorithm
consists of three threads which run concurrently in each agent:
communication (lines 1-6), navigation (lines 7-13), and internal state
transition (lines 14-26).  An outline of the local variables used for
these threads is shown in Tables~\ref{tab:dfcd_local_vars}
and~\ref{tab:dfcd_cell_fields}.  The communication thread tracks the
internal states of all an agent's visibility neighbors.  One could
design a custom communication protocol for the deployment which would
make more efficient use of communication bandwidth, however, we find
it simplifies the presentation to assume agents have direct access to
their visibility neighbors' internal states via the data structure
Neighbor\_Data$^{[i]}$.  The navigation thread has the agent follow,
at maximum velocity $u_{\rm max}$, a queue of waypoints called
Route$^{[i]}$ as long as the internal state component $c_{\xi_{\rm
    proxied}}^{[i]}$.Wait\_Set is empty (it is only ever nonempty for
a proxy agent and its meaning is discussed further in
Section~\ref{subsec:dfcd_proxy}).  The waypoints can be represented in
a local coordinate system established by the agent every time it
enters a new cell, e.g., a polar coordinate system with origin at the
cell's vantage point.  In the internal state transition thread, an
agent switches between {\tt lead}, {\tt proxy}, and {\tt explore}
modes.  The agent reacts to different asynchronous events depending on
what mode it is in.  We treat the details of the different mode
behaviors and corresponding subroutines in the following
Sections~\ref{subsec:dfcd_lead}, \ref{subsec:dfcd_proxy}, and
\ref{subsec:dfcd_explore}.

\begin{table}
\caption{\label{tab:dfcd_local_vars} Agent Local Variables for Distributed Deployment }
\begin{center}
\vspace{-1em}
\renewcommand{\arraystretch}{1.5}


\begin{tabular}{|c|c|l|}
\hline
Use           & Name                                  & Brief Description \\ \hline \hline

\multirow{5}{*}{Communication}  &UID$^{[i]}$ := $i$                     & agent Unique IDentifier\rule[-0.2cm]{0cm}{0.2cm} \\

 &In\_Buffer$^{[i]}$                     & FIFO queue of messages received from other agents\rule[-0.2cm]{0cm}{0.2cm} \\

 &Neighbor\_Data$^{[i]}$                 & \parbox[t]{9cm}{data structure which tracks relevant state information of visibility neighbors\rule[-0.2cm]{0cm}{0.2cm}} \\

 &state\_change\_interrupt$^{[i]}$       & \parbox[t]{9cm}{boolean, {\tt true} if and only if internal state has changed between the last and current iteration of the communication thread\rule[-0.2cm]{0cm}{0.2cm}} \\

 &new\_visible\_agent\_interrupt$^{[i]}$ & \parbox[t]{9cm}{boolean, {\tt true} if and only if a new agent became visible between the last and current iteration of the communication thread\rule[-0.2cm]{0cm}{0.2cm}} \\ \hline

\multirow{2}{*}{Navigation} &Route$^{[i]}$   & FIFO queue of waypoints\rule[-0.2cm]{0cm}{0.2cm} \\
 &$p^{[i]}$, $\dot{p}^{[i]}$, $u$            & position, velocity, and velocity input \\ \hline

\multirow{5}{*}{Internal State} &mode$^{[i]}$                        & \parbox[t]{9cm}{agent mode takes a value {\tt lead}, {\tt proxy}, or {\tt explore}\rule[-0.2cm]{0cm}{0.2cm}} \\ 

&Vantage\_Points$^{[i]}$ := $( p_{\xi_1}^{[i]}, p_{\xi_2}^{[i]} )$      & \parbox[t]{9cm}{ vantage points used in {\tt lead} mode for distributed representation of $\Tp$; may have size 0, 1, or 2; each $p_\xi$ may be labeled either {\tt sparse} or {\tt nonsparse}\rule[-0.2cm]{0cm}{0.2cm}} \\

&Cells$^{[i]}$ := $( c_{\xi_1}^{[i]}, c_{\xi_2}^{[i]} )$                & \parbox[t]{9cm}{cells used in {\tt lead} mode for distributed representation of $\Tp$; may have size 0, 1, or 2\rule[-0.2cm]{0cm}{0.2cm}; cell fields shown in Tab.~\ref{tab:dfcd_cell_fields}} \\


&$c_{\xi_{\rm proxied}}^{[i]}$                                         & \parbox[t]{9cm}{used in {\tt proxy} mode as local copy of cell being proxied\rule[-0.2cm]{0cm}{0.2cm}} \\

&$\xi_{\rm current}^{[i]}$, $\xi_{\rm last}^{[i]}$                      & \parbox[t]{9cm}{PTVUIDs of current and last $\Tp$ vertices visited in depth-first search; used in {\tt explore} mode to navigate $\Tp$ \rule[-0.2cm]{0cm}{0.2cm}} \\ \hline


\end{tabular}


\renewcommand{\arraystretch}{1.0}
\end{center}
\end{table}
\begin{table}
\caption{\label{tab:dfcd_cell_fields} Cell Data Fields for Distributed Deployment }
\begin{center}
\vspace{-1em}
\renewcommand{\arraystretch}{1.5}
\begin{tabular}{|c|l|}
\hline
Name                        & Brief Description \\ \hline \hline

$\xi$                       & PTVUID (Partition Tree Vertex Unique IDentifier)\rule[-0.2cm]{0cm}{0.2cm} \\ 

$c_\xi$.Boundary             & \parbox[t]{9cm}{polygonal boundary with each gap edge labeled either as
                                                    {\tt parent}, {\tt child}, {\tt unexplored}, or {\tt phantom\_wall}; 
                                                    child gap edges may be additionally labeled with an agent UID if that agent has been assigned as leader of that gap edge\rule[-0.2cm]{0cm}{0.2cm}} \\

$c_\xi$.status               & \parbox[t]{9cm}{cell status may take a value {\tt retracting}, {\tt contending}, or {\tt permanent} \rule[-0.2cm]{0cm}{0.2cm}} \\

$c_\xi$.proxy\_uid           & \parbox[t]{9cm}{UID of agent assigned to proxy $c_\xi$; takes value $\emptyset$ if no proxy has been assigned\rule[-0.2cm]{0cm}{0.2cm}} \\

$c_\xi$.Wait\_Set            & \parbox[t]{9cm}{set of PTVUIDs used by proxy agents to decide when they should wait for another cell's proxy tour to complete before deconfliction can occur, thus preventing race conditions\rule[-0.2cm]{0cm}{0.2cm}} \\ \hline



\end{tabular}
\renewcommand{\arraystretch}{1.0}
\end{center}
\end{table}
\begin{table}
\caption{\label{tab:dfcd} Distributed Deployment Algorithm }
\vspace{-3em}

\begin{quote}
{\small
      
{\rule[0em]{\linewidth}{1pt}}
DISTRIBUTED\_DEPLOYMENT()

\begin{algorithmic}[1]

\vspace{0.3em} 

%

\STATE \COMMENT{ Communication Thread }
\WHILE{ {\tt true} }     
     \STATE in\_message $\leftarrow$ In\_Buffer$^{[i]}$.PopFirst();
     \STATE update Neighbor\_Data$^{[i]}$ according to in\_message;
     \IF{ state\_change\_interrupt$^{[i]}$ {\bf or} visible\_agent\_interrupt$^{[i]}$ }
          \STATE broadcast internal state information;
     \ENDIF
\ENDWHILE

\vspace{0.5em} 

\STATE \COMMENT{ Navigation Thread }
\WHILE{ {\tt true} }
     \WHILE{ Route$^{[i]}$ is nonempty {\bf and} $p^{[i]}$ $\neq$ Route$^{[i]}$.First() {\bf and} $c_{\xi_{\rm proxied}}^{[i]}$.Wait\_Set is empty }
          \STATE $u^{[i]} \leftarrow$ velocity with magnitude $\umax$ and direction towards Route$^{[i]}$.First();
     \ENDWHILE
     \STATE $u^{[i]} \leftarrow 0$;
     \IF{ $p^{[i]}$ == Route$^{[i]}$.First() }
         \STATE Route$^{[i]}$.PopFirst();
     \ENDIF
\ENDWHILE

\vspace{0.5em} 

\STATE \COMMENT{ Internal State Transition Thread }
\WHILE{ {\tt true} }   
     \IF{ mode$^{[i]}$ == {\tt lead} }
          \STATE ATTEMPT\_CELL\_CONSTRUCTION(); \COMMENT{ See Tab.~\ref{tab:dfcd_attempt_cell_construction} }
          \STATE LEAD(); \COMMENT{ See Tab.~\ref{tab:dfcd_lead} }
          \STATE PROPAGATE\_SPARSE\_VANTAGE\_POINT\_INFORMATION(); \COMMENT{ See Tab.~\ref{tab:dfcd_propagate_sparse_vantage_point_information} }
     \ELSIF{ mode$^{[i]}$ == {\tt proxy} }
          \IF{ $c_{\rm proxied}$.status == {\tt retracting} }
               \STATE PROXY\_RETRACTING\_CELL(); \COMMENT{ See Tab.~\ref{tab:dfcd_proxy_retracting} }
          \ELSIF{ $c_{\rm proxied}$.status == {\tt contending} }
               \STATE PROXY\_CONTENDING\_CELL(); \COMMENT{ See Tab.~\ref{tab:dfcd_proxy_contending} }
          \ENDIF
     \ELSIF{ mode$^{[i]}$ == {\tt explore} }
          \STATE EXPLORE(); \COMMENT{ See Tab.~\ref{tab:dfcd_explore} }  
     \ENDIF

\ENDWHILE

\end{algorithmic}
\vspace{-0.5em}

{\rule[0.3em]{\linewidth}{0.5pt}}

}
\end{quote}

\end{table}
%
\begin{table}
\caption{\label{tab:dfcd_attempt_cell_construction} Distributed Deployment Subroutine }
\vspace{-3em}

\begin{quote}
{\small
      
{\rule[0em]{\linewidth}{1pt}}
ATTEMPT\_CELL\_CONSTRUCTION()

\begin{algorithmic}[1]


\IF{ there is a vantage point $p_\xi$ in Vantage\_Points$^{[i]}$ for which no cell in Cells$^{[i]}$ has yet been constructed
    \\ \ \ \ \ {\bf and} $p^{[i]} == p_\xi$ }
     \IF{ Neighbor\_Data$^{[i]}$ shows a cell $c_{\xi'}$ such that $c_{\xi'}$.proxy\_uid == $i$  }
          \STATE \COMMENT{ Proxy for another leader }
          \STATE mode$^{[i]} \leftarrow$ {\tt proxy}; Route$^{[i]} \leftarrow$ tour which traverses all gap edges of $c_\xi'$ and returns to $p_\xi$;

     \ELSIF{ Neighbor\_Data$^{[i]}$ shows any contending or permanent cell $c_{\xi'}$ which contains the gap edge associated with $\xi$
             \\ \ \ \ \ \ \ \ \ {\bf and} $\xi'$ is not the parent PTVUID of $\xi$ }
          \STATE \COMMENT{ Delete partition tree vertex if there is not at least one unique triangle }
          \STATE delete $(p_\xi, c_\xi)$;
          \IF{ Cells$^{[i]}$ is empty }
               \STATE mode$^{[i]} \leftarrow$ {\tt explore}; swap $\xi^{[i]}_{\rm last}$ and $\xi^{[i]}_{\rm current}$;
          \ELSIF{ Cells$^{[i]}$ contains exactly one cell }
               \STATE Route$^{[i]} \leftarrow$ straight path to $p^{[i]}_{\xi_1}$;
          \ENDIF

     \ELSIF{ Neighbor\_Data$^{[i]}$ shows no other agent constructing a cell $c_{\xi'}$ where $\xi' < \xi$ }
          \STATE \COMMENT{ Compute initial cell }
          \STATE $c_\xi \leftarrow \VVver(p_\xi)$;
          \STATE truncate $c_\xi$ such that only the portion remains which is across its parent gap edge; 
          \FOR{ each gap edge $g'$ of $c_\xi$ }     
               \IF{ $g'$ is the parent gap edge }
                    \STATE label $g'$ as {\tt parent} in $c_\xi$;
               \ELSE
                    \STATE label $g'$ as {\tt unexplored} in $c_\xi$;
               \ENDIF
          \ENDFOR
          \STATE insert $c_\xi$ into Cells$^{[i]}$;
     \ENDIF

\ENDIF

\end{algorithmic}
\vspace{-0.5em}

{\rule[0.3em]{\linewidth}{0.5pt}}

}
\end{quote}

\end{table}
\begin{table}
\caption{\label{tab:dfcd_lead} Distributed Deployment Subroutine }
\vspace{-3em}

\begin{quote}
{\small
      
{\rule[0em]{\linewidth}{1pt}}
LEAD()

\begin{algorithmic}[1]


\COMMENT{ Task assignments }

\IF{ Cells$^{[i]}$ contains only a single permanent cell $c^{[i]}_{\xi_1}$
\\ \ \ \ \ {\bf and} $c^{[i]}_{\xi_1}$ is triangle with one unexplored gap edge $g$
\\ \ \ \ \ {\bf and} $g$ has not been assigned a leader }
     \STATE \COMMENT{ Assign self a secondary vertex at child of primary vertex }
     \STATE $p^{[i]}_{\xi_2} \leftarrow p_\xi$;
     \STATE Route$^{[i]}$ $\leftarrow$ straight line path to $p_\xi$;
     \STATE label $g$ on $c^{[i]}_{\xi_1}$ as {\tt child} and as having leader $i$;

\ELSIF{ Cells$^{[i]}$ contains cell $c_\zeta$ with double child vantage point $p_\xi = p_{\xi'}$ where $\xi < \xi'$
\\ \ \ \ \ {\bf and} Neighbor\_Data$^{[i]}$ contains a leader agent $j$ with $c_\xi$ in Cells$^{[j]}$
\\ \ \ \ \ {\bf and} $p_\xi$ is labeled {\tt sparse} 
\\ \ \ \ \ {\bf and} gap edge $g$ associated with $p_{\xi'}$ is unexplored }
     \STATE \COMMENT{ Assign other leader a secondary vertex at double vantage point }
     \STATE label $g$ on $c_\zeta$ as {\tt child} and having leader $j$;

\ELSIF{ Neighbor\_Data$^{[i]}$ shows explorer agent $j$ such that $c_\xi = c_{\xi^{[j]}_{\rm current}}$ is permanent in Cells$^{[i]}$ }
     \STATE $\xi'$ $\leftarrow$ PTVUID of next vertex in \emph{depth-first ordering};
     \IF{ there is an unexplored gap edge $g$ of $c_\xi$ 
     \\ \ \ \ \  {\bf and} $\bigl ($ vantage point $p_{\xi'}$ associated with $g$ is single vantage point 
     \\ \ \ \ \ \ \ \ \ {\bf or} double vantage point with colocated vantage point {\tt nonsparse} in Neighbor\_Data$^{[i]}$ $\bigr )$ }
          \STATE \COMMENT{ Assign explorer to become leader of child vertex }
          \STATE label $g$ in $c_\xi$ as {\tt child} and having leader $j$;
     \ENDIF
\ENDIF

\IF{ Neighbor\_Data$^{[i]}$ contains an explorer agent $j$
\\ \ \ \ \ {\bf and} Cells$^{[i]}$ contains a cell $c_\xi = c_{\xi^{[j]}_{\rm current}}$ with $c_\xi$.status $\neq$ {\tt permanent}
\\ \ \ \ \ {\bf and} $c_\xi$.proxy\_uid == $\emptyset$ }
     \STATE \COMMENT{ Assign explorer as proxy }
     \STATE $c_\xi$.proxy\_uid $\leftarrow$ $j$;
\ELSIF{ Neighbor\_Data$^{[i]}$ contains a leader agent $j$ with Cells$^{[j]}$ empty
\\ \ \ \ \ {\bf and} Cells$^{[i]}$ contains a retracting cell $c_\xi$ {\bf and} $c_\xi$.proxy\_uid == $\emptyset$ }
     \STATE \COMMENT{ Assign leader as proxy }
     \STATE $c_\xi$.proxy\_uid $\leftarrow$ $j$;
\ENDIF


\IF{ Neighbor\_Data$^{[i]}$ contains a child gap edge $g$ with agent $i$ labeled as its leader
     \\ \ \ \ \ {\bf and} the associated vantage point $p_\xi$ is not in Vantage\_Points$^{[i]}$ }
     \STATE \COMMENT{ Accept leadership of second cell at double vantage point }
     \STATE $p^{[i]}_{\xi_2} \leftarrow p_\xi$;
\ENDIF

\end{algorithmic}

(continued)

}
\end{quote}

\end{table}
\begin{table}
\vspace{-3em}

\begin{quote}
{\small
      
{(continuation)}

\begin{algorithmic}[1]
\addtocounter{ALC@line}{22}


\STATE \COMMENT{ React to deconfliction events }

\IF{ a cell $c_\xi$ in Cells$^{[i]}$ corresponds to a cell $c_{\xi_{\rm proxied}}^{[j]}$ in Neighbor\_Data$^{[i]}$ }  
     \IF{ $c_{\xi_{\rm proxied}}^{[j]}$ has been truncated at a permanent cell }
          \STATE perform the same truncation on $c_\xi$;
     \ENDIF
     \IF{ $c_\xi$.Wait\_Set $\neq$ $c_{\xi_{\rm proxied}}^{[j]}$.Wait\_Set }
          \STATE $c_\xi$.Wait\_Set $\leftarrow$ $c_{\xi_{\rm proxied}}^{[j]}$.Wait\_Set;
     \ENDIF
\ENDIF

\IF{ Neighbor\_Data$^{[i]}$ shows a proxy has deleted a cell corresponding to $c_\xi$ in Cells$^{[i]}$ {\bf or} $\bigl ($ Neighbor\_Data$^{[i]}$ shows contending cell $c_{\xi_{\rm proxied}}^{[j]}$ in \emph{branch conflict} with contending cell $c_\xi$ in Cells$^{[i]}$ {\bf and} $\xi_{\rm proxied}^{[j]} < \xi$ $\bigr )$ }
     \IF{ Cells$^{[i]}$ contains exactly one cell }
          \STATE delete $(p^{[i]}_{\xi_1}, c^{[i]}_{\xi_1})$; mode$^{[i]} \leftarrow$ {\tt explore};
     \ELSIF{ Cells$^{[i]}$ contains two cells }
          \STATE delete $(p^{[i]}_{\xi_2}, c^{[i]}_{\xi_2})$; Route$^{[i]} \leftarrow$ straight path to $p^{[i]}_{\xi_1}$;
     \ENDIF

\ENDIF

\IF{ Neighbor\_Data$^{[i]}$ shows a cell was deleted at gap edge $g$ of cell $c_\xi$ in Cells$^{[i]}$ }
     \STATE label $g$ as {\tt phantom\_wall} in $c_\xi$;
\ENDIF

\IF{ Neighbor\_Data$^{[i]}$ shows a proxy tour was successfully completed without deletion for a cell $c_\xi$ in Cells$^{[i]}$ }
     \STATE advance $c_\xi$.status; $c_\xi$.proxy\_uid $\leftarrow \emptyset$;
\ENDIF

\end{algorithmic}
\vspace{-0.5em}

{\rule[0.3em]{\linewidth}{0.5pt}}

}
\end{quote}

\end{table}
\begin{table}
\caption{\label{tab:dfcd_propagate_sparse_vantage_point_information} Distributed Deployment Subroutine }
\vspace{-3em}

\begin{quote}
{\small
      
{\rule[0em]{\linewidth}{1pt}}
PROPAGATE\_SPARSE\_VANTAGE\_POINT\_INFORMATION()

\begin{algorithmic}[1]



\STATE \COMMENT{ Label a vantage point in Vantage\_Points$^{[i]}$ as {\tt sparse} or {\tt nonsparse} }
\IF{ there is an unlabeled vantage point $p_\xi$ in Vantage\_Points$^{[i]}$ with permanent cell $c_\xi$ in Cells$^{[i]}$
 \\ \ \ \ \ {\bf and} $\bigl ($ $(p_\xi, c_\xi)$ is a leaf {\bf or} Cells$^{[i]}$ and Neighbor\_Data$^{[i]}$ show all child vantage points have been labeled $\bigr )$}
     \IF{ $|V_{c_\xi}| == 3$ {\bf and} Cells$^{[i]}$ or Neighbor\_Data$^{[i]}$ shows a child vantage point labeled {\tt sparse} }
          \STATE label $p_\xi$ as {\tt nonsparse};
     \ELSE
          \STATE label $p_\xi$ as {\tt sparse};
     \ENDIF
\ENDIF

\STATE \COMMENT{ Acquire a nonsparse vertex from an agent higher in the partition tree}
\IF{ Cells$^{[i]}$ contains exactly one cell $c_\xi$ with $p_\xi$ labeled {\tt sparse} {\bf and} $p^{[i]}$ == $p_\xi$
     \\ \ \ \ \ {\bf and} Neighbor\_Data$^{[i]}$ shows a cell $c_\zeta$ which is the parent of $c_\xi$ {\bf and} $p_\zeta$ is labeled {\tt nonsparse} }
     \STATE insert $c_\zeta$ into Cells$^{[i]}$ and $p_\zeta$ into Vantage\_Points$^{[i]}$;
\ENDIF

\STATE \COMMENT{ Surrender a nonsparse vertex to an agent lower in the partition tree } 
\IF{ Neighbor\_Data$^{[i]}$ shows a leader agent $j$ with $p^{[j]}_{\xi_1}$ labeled {\tt sparse}
\\ \ \ \ \ {\bf and} $c^{[i]}_{\xi_2} == c^{[j]}_{\xi_2}$ {\bf and} $\xi^{[j]}_2$ is the parent PTVUID of $\xi^{[i]}_1$ }
     \STATE clear $p^{[i]}_{\xi_2}$ and $c^{[i]}_{\xi_2})$; Route$^{[i]} \leftarrow$ straight path to $p^{[i]}_{\xi_1}$;
\ENDIF

\end{algorithmic}
\vspace{-0.5em}

{\rule[0.3em]{\linewidth}{0.5pt}}

}
\end{quote}

\end{table}
\begin{table}
\caption{\label{tab:dfcd_proxy_retracting} Distributed Deployment Subroutine }
\vspace{-3em}

\begin{quote}
{\small

{\rule[0em]{\linewidth}{1pt}}
PROXY\_RETRACTING\_CELL()

\begin{algorithmic}[1]


\IF{ Route$^{[i]}$ is nonempty } 

     \STATE \COMMENT{ Truncate $c_{\xi_{\rm proxied}}$ at permanent cell }
     \IF{ Neighbor\_Data$^{[i]}$ shows permanent cell $c_\xi$ in \emph{branch conflict} with $c^{[i]}_{\xi_{\rm proxied}}$ }
          \STATE truncate $c^{[i]}_{\xi_{\rm proxied}}$ at $c_\xi$; 
     \ENDIF

     \STATE \COMMENT{ Prevent race conditions and deadlock }
     \IF{ Neighbor\_Data$^{[i]}$ shows contending cell $c_\xi$ in \emph{branch conflict} with $c^{[i]}_{\xi_{\rm proxied}}$
          \\ \ \ \ \ \ {\bf and} $c_\xi$.proxy\_uid $\neq$ $\emptyset$ {\bf and} $\bigl ($ $\xi^{[i]}_{\rm proxied}$ $\notin$ $c_\xi$.Wait\_Set {\bf or} $\xi < \xi^{[i]}_{\rm proxied}$ $\bigr )$ }

          \STATE $c^{[i]}_{\xi_{\rm proxied}}$.Wait\_Set $\leftarrow$ $c^{[i]}_{\xi_{\rm proxied}}$.Wait\_Set $\cup$ $\xi$;
     \ELSE
          \STATE $c^{[i]}_{\xi_{\rm proxied}}$.Wait\_Set $\leftarrow$ $c^{[i]}_{\xi_{\rm proxied}}$.Wait\_Set $\setminus$ $\xi$;
     \ENDIF

\ELSIF{ Route$^{[i]}$ is empty }  
     \STATE \COMMENT{ End proxy tour and enter previous mode }
     \IF{ Vantage\_Points$^{[i]}$ is empty }
          \STATE mode$^{[i]} \leftarrow$ {\tt explore}; 
     \ELSE 
          \STATE mode$^{[i]} \leftarrow$ {\tt lead}; 
     \ENDIF
     \STATE clear $c^{[i]}_{\xi_{\rm proxied}}$;

\ENDIF

\end{algorithmic}
\vspace{-0.5em}

{\rule[0.3em]{\linewidth}{0.5pt}} 

}
\end{quote}

\end{table}
\begin{table}
\caption{\label{tab:dfcd_proxy_contending} Distributed Deployment Subroutine }
\vspace{-3em}

\begin{quote}
{\small

{\rule[0em]{\linewidth}{1pt}}
PROXY\_CONTENDING\_CELL()

\begin{algorithmic}[1]


\IF{ Route$^{[i]}$ is nonempty {\bf and} the parent gag edge of $c^{[i]}_{\xi_{\rm proxied}}$ is not phantom wall }
     
     \STATE \COMMENT{ Shoot-out with other contending cells }
     \IF{ $\bigl ($ Neighbor\_Data$^{[i]}$ shows contending cell $c_\xi$ in \emph{branch conflict} with $c^{[i]}_{\xi_{\rm proxied}}$ {\bf and} $\xi < \xi^{[i]}_{\rm proxied}$ $\bigr )$
          \\ \ \ \ \ \ \ {\bf or} Neighbor\_Data$^{[i]}$ shows a phantom wall coinciding with parent gap edge of $c^{[i]}_{\xi_{\rm proxied}}$ }
          \STATE delete $c^{[i]}_{\xi_{\rm proxied}}$; mode$^{[i]} \leftarrow$ {\tt explore};
     \ENDIF

     \STATE \COMMENT{ Prevent race conditions and deadlock }
     \IF{ Neighbor\_Data$^{[i]}$ shows retracting cell $c_\xi$ in \emph{branch conflict} with $c^{[i]}_{\xi_{\rm proxied}}$
          \\ \ \ \ \ \ {\bf and} $c_\xi$.proxy\_uid $\neq$ $\emptyset$ {\bf and} $\bigl ($ $\xi^{[i]}_{\rm proxied}$ $\notin$ $c_\xi$.Wait\_Set {\bf or} $\xi < \xi^{[i]}_{\rm proxied}$ $\bigr )$ }
          \STATE $c^{[i]}_{\xi_{\rm proxied}}$.Wait\_Set $\leftarrow$ $c^{[i]}_{\xi_{\rm proxied}}$.Wait\_Set $\cup$ $\xi$;
     \ELSE
          \STATE $c^{[i]}_{\xi_{\rm proxied}}$.Wait\_Set $\leftarrow$ $c^{[i]}_{\xi_{\rm proxied}}$.Wait\_Set $\setminus$ $\xi$;
     \ENDIF

\ELSIF{ Route$^{[i]}$ is empty } 
     \STATE \COMMENT{ End proxy tour and become explorer }
     \STATE mode$^{[i]} \leftarrow$ {\tt explore}; clear $c_{\xi_{\rm proxied}}$;
\ENDIF

\end{algorithmic}
\vspace{-0.5em}

{\rule[0.3em]{\linewidth}{0.5pt}} 

}
\end{quote}

\end{table}
\begin{table}
\caption{\label{tab:dfcd_explore} Distributed Deployment Subroutine }
\vspace{-3em}

\begin{quote}
{\small
      
{\rule[0em]{\linewidth}{1pt}}
EXPLORE()

\begin{algorithmic}[1]


\IF{ Neighbor\_Data$^{[i]}$ shows a permanent cell $c_\xi$ where $\xi$ == $\xi^{[i]}_{\rm current}$ }
     \STATE $\xi'$ $\leftarrow$ PTVUID of next vertex in \emph{depth-first ordering};
     \IF{ gap edge $g$ at $\xi'$ has already been assigned a leader }
          \STATE \COMMENT{ Continue exploring }
          \STATE $\xi^{[i]}_{\rm last} \leftarrow \xi^{[i]}_{\rm current}$; $\xi^{[i]}_{\rm current} \leftarrow \xi'$;
          \STATE Route$^{[i]} \leftarrow$ local shortest path to midpoint of $g$ through $c_\xi$;
     \ELSIF{ gap edge $g$ at $\xi'$ has agent $i$ labeled as its leader }
          \STATE \COMMENT{ Become leader }
          \STATE mode$^{[i]}$ $\leftarrow$ {\tt lead}; $p^{[i]}_{\xi_1} \leftarrow p_{\xi'}$;
          \STATE Route$^{[i]} \leftarrow$ local shortest path to $p_{\xi'}$ through $c_\xi$;
     \ENDIF

\ELSIF{ Neighbor\_Data$^{[i]}$ shows a cell $c_\xi$ such that $c_\xi$.proxy\_uid == $i$  }
     \STATE \COMMENT{ Become proxy }
     \STATE mode$^{[i]} \leftarrow$ {\tt proxy}; $c^{[i]}_{\xi_{\rm proxied}} \leftarrow c_\xi$;
     \STATE Route$^{[i]} \leftarrow$ tour which traverses all gap edges of $c_\xi$ and returns to parent gap edge;
\ENDIF

\IF{ Neighbor\_Data$^{[i]}$ shows $c_{\xi^{[i]}_{\rm current}}$ has been deleted } 
     \STATE \COMMENT{ Move up partition tree in reaction to deleted cell }
     \STATE Route$^{[i]} \leftarrow$ local shortest path towards $c_{\xi_{\rm last}}$; swap $\xi^{[i]}_{\rm last}$ and $\xi^{[i]}_{\rm current}$;
\ENDIF

\end{algorithmic}
\vspace{-0.5em}

{\rule[0.3em]{\linewidth}{0.5pt}}

}
\end{quote}
\end{table}

\subsection{Leader Behavior}
\label{subsec:dfcd_lead}

The {\tt lead} portion of the internal state transition thread (lines
16-19 of Table~\ref{tab:dfcd}) consists of three subroutines:
ATTEMPT\_CELL\_CONSTRUCTION(), LEAD(), and
PROPAGATE\_SPARSE\_VANTAGE\_POINT\_INFORMATION().  In
ATTEMPT\_CELL\_CONSTRUCTION()
(Table~\ref{tab:dfcd_attempt_cell_construction}), the leader agent
attempts to construct a cell, say $c_\xi$, whenever it first arrives
at $p_\xi$.  In order to guarantee an upper bound on the number of
agents required by the deployment
(Theorem~\ref{thm:dfcd_convergence}), the leader must enforce that any
cell it adds to $\Tp$ contains at least one unique triangle which is
not in any other cell of the distributed $\Tp$ representation.  This
can be accomplished by the leader first looking at its Neighbor\_Data
to see if the parent gap edge, call it $g$, is contained in the cell
of any neighbor other than the parent.  If not, then the existence of
a unique triangle is guaranteed because cell vertices always coincide
with environment vertices.  In that case the agent safely initializes
the cell to {\tt retracting} status and waits for a proxy agent to
help it advance the cell's status towards {\tt permanent}.  If,
however, $g$ is contained in a neighbor cell other than the parent,
then the leader may have to either switch to proxy mode to proxy for
another leader in line of sight (if the candidate cell is primary), or
else wait for the other cell to be proxied (if the candidate cell is
secondary).  If the agent determines that a contending or permanent
cell other than the parent contains $g$, then it deletes the cell and
a phantom wall is labeled.

In LEAD() (Table~\ref{tab:dfcd_lead}), the agent already has
initialized cell(s) in its memory.  Being responsible for cells means
that the leader agent may have to assign tasks.  The assignment may be
of an explorer to become a leader of a child vertex, of an explorer to
become a proxy, of a leader to become a proxy, of itself to lead a
secondary $\Tp$ vertex which is the child of its primary vertex (this
happens when the primary vertex is a triangle), or of another leader
to a secondary vertex at a double vantage point.  Note that in making
the assignments, all vantage points are selected according to the same
\emph{parity-based vantage point selection scheme} used in the
incremental partition algorithm of
Sec.~\ref{sec:incremental_partition}.  So that the distributed
representation of $\Tp$ remains consistent, a leader must also react
to several deconfliction events.  If a proxy truncates the boundary of
a retracting cell, deletes a contending cell, advances the status of a
cell, or adds/removes PTVUIDs to a cell's Wait\_Set, then the
corresponding leader of that cell must do the same.  In fact, whenever
two agents (either proxies or leaders) communicate and their
contending cells are in branch conflict, the cell with lower PTVUID
will be deleted.  Every such cell deletion results in a phantom wall
being marked in the parent cell.  Although it is not stated
explicitely in the pseudocode, note that when a cell is deleted the
leader must wait briefly at the cell's vantage point until any agent
that was proxying comes back to the parent cell; otherwise the proxy
could lose line of sight with the rest of the network.  If a proxy
tour is completed successfully without cell deletion, then the cell
status is advanced towards {\tt permanent}.

By settling only to sparse vantage points, fewer agents are needed to
guarantee full coverage.  This is accomplished by the behavior in
PROPAGATE\_SPARSE\_VANTAGE\_POINT\_INFORMATION()
(Table~\ref{tab:dfcd_propagate_sparse_vantage_point_information})
where agents swap permanent cells with other leaders in such a way
that the information about which vantage points are sparse is
propagated up $\Tp$ whenever a leaf is discovered.  Each cell swap
involves an acquisition by one agent (lines 7-9) and a corresponding
surrender by another (lines 10-12).





\subsection{Proxy Behavior}
\label{subsec:dfcd_proxy}

The {\tt proxy} portion of the internal state transition thread on
lines 20-24 of Table~\ref{tab:dfcd} runs one of two subroutines
depending on the status of the proxied cell: PROXY\_RETRACTING\_CELL()
and PROXY\_CONTENDING\_CELL().  Suppose an agent $i$ is proxying for a
cell $c_\xi$ in leader agent $j$'s memory.  Then agent $i$ keeps a
local copy of $c_\xi$ in $c^{[i]}_{\xi_{\rm proxied}}$ and modifies it
during the proxy tour.  Agent $j$ updates $c_\xi$ to match
$c^{[i]}_{\xi_{\rm proxied}}$ whenever a change occurs.  
%
In PROXY\_RETRACTING\_CELL() (Table~\ref{tab:dfcd_proxy_retracting}),
agent $i$ traverses the gap edges of $c^{[i]}_{\xi_{\rm proxied}}$
while truncating the cell boundary at any encountered permanent cells
in branch conflict.  The goal is for the retracting proxied cell to
not be in branch conflict with any permanent cells by the end of the
proxy tour when its status is advanced to {\tt contending}.  If agent
$i$ encounters a contending cell, say $c_{\xi'}$, and the criteria on
line 6 are satisfied, then agent $i$ must pause its proxy tour, i.e.,
pause motion until $c_{\xi'}$ becomes permanent or deleted.  If the
proxy were not to pause, then it would run the risk of the contending
cell becoming permanent after the opportunity for the proxy to perform
truncation had already passed.  The pausing is accomplished by adding
$\xi'$ to the cell field $c^{[i]}_{\xi_{\rm proxied}}$.Wait\_Set read
by the navigation thread.  Once the proxy tour is over, the leader of
the proxied cell advances the cell's status to {\tt contending} and
the proxy agent enters its previous mode, either explore or lead.

%
In PROXY\_CONTENDING\_CELL() (Table~\ref{tab:dfcd_proxy_contending}),
the goal is for the contending proxied cell to not be in branch
conflict with any other contending cells by the end of the proxy tour
if its status is to be advanced to {\tt permanent}. To this end, agent
$i$ traverses the gap edges of $c^{[i]}_{\xi_{\rm proxied}}$ while
comparing $\xi^{[i]}_{\rm proxied}$ with the PTVUID of every
encountered contending cell in branch conflict with $c^{[i]}_{\xi_{\rm
    proxied}}$.  If a contending cell with PTVUID less than
$\xi^{[i]}_{\rm proxied}$ is encountered, then the proxied cell is
deleted (signified by labeling a phantom wall) and agent $i$ heads
straight back to the parent gap edge where it will end the proxy tour
and enter {\tt explore} mode.  If agent $i$ encounters a retracting
cell, say $c_{\xi'}$, and the criteria on line 6 are satisfied, then
agent $i$ must pause its proxy tour, i.e., pause motion, until
$c_{\xi'}$ becomes contending or truncated out of branch conflict.  If
the proxy were not to pause, then it would run the risk of the
retracting cell becoming contending after the opportunity for the
proxy to perform deconfliction had already passed.  The pausing is
accomplished by adding $\xi'$ to the cell field $c^{[i]}_{\xi_{\rm
    proxied}}$.Wait\_Set read by the navigation thread.  Finally, if a
contending cell with PTVUID less than $\xi^{[i]}_{\rm proxied}$ is
never encountered, then the leader of the proxied cell advances the
cell's status to {\tt permanent} and the proxy agent enters {\tt
  explore} mode.

Note that the use of PTVUID total ordering
(Definition~\ref{defn:ptvuid_total_ordering}) on line 6 of
PROXY\_RETRACTING\_CELL() and line 3 and 6 of
PROXY\_CONTENDING\_CELL() precludes the possibility of both (1)
\emph{race conditions} in which the status of cells is advanced before
the proper branch deconflictions have taken place, and (2)
\emph{deadlock situations} where contending and retracting cells are
indefinitely waiting for each other.





\subsection{Explorer Behavior}
\label{subsec:dfcd_explore}

The {\tt explore} portion of the internal state transition thread on
lines 25-26 of Table~\ref{tab:dfcd} consists of a single subroutine
EXPLORE() shown in Table~\ref{tab:dfcd_explore}. Of all agent modes,
{\tt explore} behavior is the simplest because all the agent has to do
is navigate $\Tp$ in depth-first order (see
Fig.~\ref{fig:df_ordering_micro} and \ref{fig:df_ordering_macro})
until a leader agent assigns them to become a leader at an unexplored
gap edge or to perform a proxy task.  The local shortest paths between
cells (lines 6, 10, and 17) can be computed quickly and easily by the
visibility graph method \cite{Nil69}.  If the current cell that an
explorer agent is visiting is ever deleted because of branch
deconfliction, the explorer simply moves up $\Tp$ and continues
depth-first searching.  By having each agent use a different gap edge
ordering for the depth-first search, the deployment tends to explore
many partition tree branches in parallel and thus converge more
quickly.  In our simulations (Sec.~\ref{subsec:simulation}), we had
each agent cyclically shift their gap edge ordering by their UID,
subject to the following restriction important for proving an upper
bound on number of required agents in
Theorem~\ref{thm:dfcd_convergence}.

\begin{remark}[Restriction on Depth-First Orderings]
\label{rm:df_restriction}
Each agent in an execution of the distributed deployment may search
$\Tp$ depth-first using any child ordering as long as every pair of
child vertices adjacent at a double vantage point are visited in the
same order by every agent.
\end{remark}

% ...
%

\subsection{Performance Analysis}
\label{subsec:convergence}

The convergence properties of the Distributed Depth-First Connected
Deployment Algorithm of Table~\ref{tab:dfcd} are captured in the
following theorems.

\begin{theorem}[Convergence]
  \label{thm:dfcd_convergence}
  Suppose that $N$ agents are initially colocated at a common point
  $p_\emptyset \in V_\EE$ of a polygonal environment $\EE$ with $n$
  vertices and $h$ holes.  If the agents operate according to the
  Depth-First Connected Deployment Algorithm of Table~\ref{tab:dfcd},
  then
  \begin{enumerate}
  \item the agents' visibility graph $\Gvis{\EE}(P)$ consists of a
    single connected component at all times,
  \item there exists a finite time $t^*$, such that for all times
    greater than $t^*$ the set of vertices in the distributed
    representation of the partition tree $\Tp$ remains fixed,
  \item if the number of agents $N \geq \lfloor \frac{n+2h-1}{2}
    \rfloor $, then for all times greater than $t^*$ every point in
    the environment $\EE$ will be visibile to some agent, and there
    will be no more than $h$ phantom walls, and
  \item if $N > \lfloor \frac{n+2h-1}{2} \rfloor $, then for all times
    greater than $t^*$ every cell in the distributed representation of
    $\Tp$ will have status {\tt permanent} and there will be precisely
    $h$ phantom walls.
  \end{enumerate}
\end{theorem}
\begin{proof}
  %

  We prove the statements in order. Nonleader agents, as we have
  defined their behavior, remain at all times within line of sight of
  at least one leader agent.  Leader agents likewise remain in the
  kernel of their cell(s) of responsibility and within line of sight
  of the leader agent responsible for the corresponding parent
  cell(s).  Given any two agents, say $i$ and $j$, a path can thus be
  constructed by first following parent-child visibility links from
  agent $i$ up to the leader agent responsible for the root, then from
  the leader agent responsible for the root down to agent $j$.  The
  agents' visibility graph must therefore consist of a single
  connected component, which is statement (i).

  For statement (ii), we argue similarly to the proof of
  Theorem~\ref{thm:incremental_partition_convergence}(i).  During the
  deployment, cells are constructed only at unexplored gap edges.  A
  cell either (1) advances though a finite number of status changes or
  (2) it is deleted during a proxy tour.  Either way, each cell is
  only modified a finite number of times and only one cell is ever
  created at any particular unexplored gap edge.  Since unexplored gap
  edges are diagonals of the environment and there are only finitely
  many possible diagonals, we conclude the set of vertices in the
  distributed representation of $\Tp$ must remain fixed after some
  finite time $t^*$.

  %
  %
  For statement (iii), we rely on an invariant: during the distributed
  deployment algorithm, at least two unique triangles can be assigned
  to every leader agent which has at least one cell of responsibility,
  other than the root cell, in its memory; at least one unique
  triangle can be assigned to the leader agent which has the root cell
  in its memory.  One of the triangles is in a leader's own cell
  (primary or secondary) and its existence is ensured by the leader
  behavior in Table~\ref{tab:dfcd_attempt_cell_construction}.  The
  second triangle is in a parent cell of a cell in the agent's memory.
  The existence of this second triangle is ensured by the depth-first
  order restriction stipulated in Remark~\ref{rm:df_restriction}
  together with the parity-based vantage point selection scheme.
  Remembering that the maximum number of triangles in any
  triangulation is $n+2h-2$ and arguing precisely as we did for the
  sparse vantage point locations in the proof of
  Theorem~\ref{thm:sparse_vantage_point_bound}(iii), we find the
  number of agents required for full coverage can be no greater than
  $\lfloor \frac{n+2h-1}{2} \rfloor $.  As in the proof of
  Theorem~\ref{thm:incremental_partition_convergence}(v), the number
  of phantom walls can be no greater than $h$ because if it where then
  some cell would be topologically isolated.

  Proof of statement (iv) is as for statement (iii), but because there
  is one extra agent and depth-first is systematic, the extra agent is
  guaranteed to eventually proxy any remaining nonpermanent cells into
  {\tt permanent} status and create phantom walls to separate all
  conflicting partition tree branches.
\end{proof}

\begin{remark}[Near Optimality without Holes]
\label{rm:no_holes_almost_optimal}
As mentioned in Sec.~\ref{sec:intro}, $(n-2)/2$ guards are always
sufficient and occasionally necessary for visibility coverage of any
polygonal environment without holes.  This means that when $h=0$, the
bound on the number of sufficient agents in
Theorem~\ref{thm:dfcd_convergence} statement (iii) differs from the
worst-case optimal bound by at most one.
\end{remark}

\begin{theorem}[Time to Convergence]
  \label{thm:dfcd_runtime}
  Let $\EE$ be an environment as in
  Theorem~\ref{thm:dfcd_convergence}.  Assume time for communication
  and processing are negligible compared with agent travel time and
  that $\EE$ has uniformly bounded diameter as $n \rightarrow \infty$.
  Then the time to convergence $t^*$ in
  Theorem~\ref{thm:dfcd_convergence} statement (ii) is
  $\mathcal{O}(n^2+nh)$.  Moreover, if the maximum perimeter length of
  any vertex-limited visibility polygon in $\EE$ is uniformly bounded
  as $n \rightarrow \infty$, then $t^*$ is $\mathcal{O}(n+h)$.
\end{theorem}
\begin{proof}
  As in the proof of Theorem~\ref{thm:dfcd_convergence}, every cell
  which is never deleted has at least one unique triangle and there
  are at most $n+2h-2$ triangles total, therefore there are at most
  $n+2h-2$ cells which are never deleted.  The maximum number of
  phantom walls ever created is $h$
  (Theorem~\ref{thm:dfcd_convergence}).
  Since cells are only ever deleted when a phantom wall is created, at
  most $h$ cells are ever deleted.  Summing the bounds on the number
  cells which are and are not deleted, we see the total number of
  cells any agent must ever visit during the distributed deployment is
  $n+2h-2+h = n+3h-2$.  Let $l_{\rm d}$ be the maximum diameter of any
  vertex-limited visibility polygon in $\EE$.  Then, neglecting time
  for proxy tours, an agent executing depth-first search on $\Tp$ will
  visit every vertex of $\Tp$ in time at most $2 u_{\rm max} l_{\rm d}
  (n+3h-2)$.  Now Let $l_{\rm p}$ be the maximum perimeter length of
  any vertex-limited visibility polygon in $\EE$.  Then the total
  amount of time agents spend on proxy tours, counting two tours for
  each cell, is $2 u_{\rm max} l_{\rm p} (n+3h-2)$.  Exploring and
  leading agents operate in parallel and at most every agent waits for
  every proxy tour, so it must be that
  \[ t^* \leq 2 u_{\rm max} ( l_{\rm p} + l_{\rm d} ) (n + 3h - 2). \]
  While the diameter of $\EE$ being uniformly bounded implies $l_{\rm
    d}$ is uniform bounded, $l_{\rm p}$ may be $\mathcal{O}(n)$.
\end{proof}

The performance of a distributed algorithm can also be measured by
agent memory requirements and the size of messages which must be
communicated.

\begin{lemma}[Memory and Communication Complexity]
\label{lm:memory_and_comm_complexity}
Let $k$ be the maximum number of vertices of any vertex-limited
visibility polygon in the environment $\EE$ and suppose $\EE$ is
represented with fixed resolution.  Then the required memory size for
an agent to run the distributed deployment algorithm is
$\mathcal{O}(Nk)$ bits and the message size is $\mathcal{O}(k)$ bits.
\end{lemma}
\begin{proof}
  The memory required by an agent for its internal state is dominated
  by its cell(s) of responsibility (of which there are at most two)
  and proxy cell (at most one).  A cell requires $\mathcal{O}(k)$
  bits, therefore the internal state requires $\mathcal{O}(k)$ bits.
  The overall amount of memory in an agent is dominated by
  Neighbor\_Data$^{[i]}$, which holds no more than $N$ internal
  states, therefore the memory requirement of an agent is
  $\mathcal{O}(Nk)$.  Agents only ever broadcast their internal state,
  therefore the message size is $\mathcal{O}(k)$.
\end{proof}

\subsection{Simulation Results}
\label{subsec:simulation}

We used C++ and the {V}isi{L}ibity library \cite{VisiLibity:08} to
simulate the Distributed Depth-First Deployment Algorithm of
Table~\ref{tab:dfcd}.  An example simulation run is shown in
Fig.~\ref{fig:dfcd_sim} for an environment with $n=41$ vertices and
$h=4$ holes.  An animation of this simulation can be viewed at {\tt
  http://motion.me.ucsb.edu/$\sim$karl/movies/dwh.mov} . To reduce
clutter, we have omitted from this larger example the agent mode and
cell status color codes used in
Fig.~\ref{fig:agent_modes_and_abstract_proxy},
\ref{fig:cell_statuses}, \ref{fig:df_ordering_micro}, and
\ref{fig:minimal_example}.  The environment was fully covered in
finite time by only 13 agents, which indeed is less than the upper
bound $\lfloor \frac{n+2h-1}{2} \rfloor = 24$ given by
Theorem~\ref{thm:dfcd_convergence}.

\subsection{Extensions}
\label{subsec:extensions}

There are several ways that the distributed deployment algorithm can
be directly extended for robustness to agent arrival, agent failure,
packet loss, and removal of an environment edge.  Robustness to agent
arrival can be achieved by having any new agents simply enter {\tt
  explore} mode, setting $\xi^{[i]}_{\rm current}$ to be the PTVUID of
the first cell they land in, and setting $\xi^{[i]}_{\rm last}$ to be
the parent PTVUID of $\xi_{\rm current}$.  The line-of-sight
connectivity guaranteed by Theorem~\ref{thm:dfcd_convergence} allows
single-agent failures to be detected and handled by having the
visibility neighbors of a failed agent move back up the partition tree
as necessary to patch the hole left by the failed agent.  For
robustness to packet loss, agents could add a receipt confirmation
and/or parity check protocol.  If a portion of the environment were
blocked off during the beginning of the deployment but then were
revealed by an edge removal (interpreted as the ``opening of a
door''), the deployment could proceed normally as long as the deleted
edge were marked as an {\tt unexplored} gap edge in the cell it
belonged to.

Less trivial extensions include (1) the use of distributed assignment
algorithms such as \cite{BJM-KMP:06,MMZ-LS-GJP:08} for guiding
explorer agents to tasks faster than depth-first search, or (2)
performing the deployment from multiple roots, i.e., when different
groups of agents begin deployment from different locations.
Deployment from multiple roots can be achieved by having the agents
tack on a root identifier to their PTVUID, however, it appears this
would increase the bound on number of agents required in
Theorem~\ref{thm:dfcd_convergence} by up to one agent per root.


%
%

\section{Conclusion}
\label{sec:conclusion}

In this article we have presented the first distributed deployment
algorithm which solves, with provable performance, the Distributed
Visibility-Based Deployment Problem with Connectivity in polygonal
environments with holes.  We began by designing a centralized
incremental partition algorithm, then obtained the distributed
deployment algorithm by asynchronous distributed emulation of the
centralized algorithm.
Given at least $\lfloor \frac{n+2h-1}{2} \rfloor $ agents in an
environment with $n$ vertices and $h$ holes, the deployment is
guaranteed to achieve full visibility coverage of the environment in
time $\mathcal{O}(n^2+nh)$, or time $\mathcal{O}(n+h)$ if the maximum
perimeter length of any vertex-limited visibility polygon in $\EE$ is
uniformly bounded as $n \rightarrow \infty$.  If $k$ is the maximum
number of vertices of any vertex-limited visibility polygon in an
environment $\EE$ represented with fixed resolution, then the required
memory size for an agent to run the distributed deployment algorithm
is $\mathcal{O}(Nk)$ bits and message size is $\mathcal{O}(k)$ bits.
The deployment behaved in simulations as predicted by the theory and
can be extended to achieve robustness to agent arrival, agent failure,
packet loss, removal of an environment edge (such as an opening door),
or deployment from multiple roots.


%

There are many interesting possibilities for future work in the area
of deployment and nonconvex coverage.  Among the most prominent are:
3D environments, dynamic environments with moving obstacles, and
optimizing different performance measures, e.g., based on continuous
instead of binary visibility, or with minimum redundancy requirements.

{\small
%

}
%


\end{document}